\documentclass{article}

\usepackage[letterpaper,margin=1in]{geometry}
\usepackage[parfill]{parskip}
\usepackage[authoryear]{natbib}

\usepackage[utf8]{inputenc} 
\usepackage[T1]{fontenc}    

\usepackage[colorlinks,citecolor=blue,urlcolor=blue,linkcolor=blue,linktocpage=true]{hyperref}

\usepackage{url}            
\usepackage{booktabs}       
\usepackage{multirow}
\usepackage{amsfonts}       
\usepackage{nicefrac}       
\usepackage{microtype}      
\usepackage[parfill]{parskip}
\usepackage{amsmath,amsthm,amssymb,bbm}
\usepackage{mathtools}
\usepackage{cases}
\usepackage{comment}
\usepackage{caption}
\usepackage{subcaption}
\usepackage{algorithm,algorithmic}
\usepackage{color}
\usepackage{appendix}
\usepackage[bottom]{footmisc}

\usepackage{xspace}
\usepackage{enumitem}
\usepackage[english]{babel}
\usepackage{cleveref}
\crefformat{equation}{(#2#1#3)}
\crefrangeformat{equation}{(#3#1#4) to~(#5#2#6)}
\crefname{equation}{}{}
\Crefname{equation}{}{}

\crefname{definition}{\textbf{definition}}{definitions}
\Crefname{definition}{Definition}{Definitions}
\crefname{assumption}{\textbf{assumption}}{assumptions}
\Crefname{assumption}{Assumption}{Assumptions}
\definecolor{maroon}{RGB}{192,80,77}

\newcommand{\explain}[2]{\underset{\mathclap{\overset{\uparrow}{#2}}}{#1}}

\newcommand\independent{\protect\mathpalette{\protect\independenT}{\perp}}
\def\independenT#1#2{\mathrel{\rlap{$#1#2$}\mkern2mu{#1#2}}}

\newtheorem{theorem}{Theorem}
\newtheorem{lemma}[theorem]{Lemma}

\newtheorem{proposition}[theorem]{Proposition}

\newtheorem{corollary}[theorem]{Corollary}
\newtheorem{definition}[theorem]{Definition}

\newtheorem{remark}[theorem]{Remark}

	\AtEndDocument{\refstepcounter{theorem}\label{finalthm}}

\newcommand{\OBP}{\textsc{ObjPert}}
\newcommand{\OPS}{\textsc{OPS}}
\newcommand{\SuffP}{\textsc{SSP}}
\newcommand{\AdaSP}{\textsc{AdaSSP}}
\newcommand{\NSGD}{\textsc{NoisySGD}}

\newcommand{\AdaOPS}{\textsc{AdaOPS}}
\newcommand{\vct}{\boldsymbol }

\newcommand{\R}{\mathbb{R}}

\newcommand{\argmin}{\mathop{\mathrm{argmin}}}

\def\E{\mathbb{E}}
\def\P{\mathbb{P}}
\def\Cov{\mathrm{Cov}}
\def\Var{\mathrm{Var}}

\def\tr{\mathrm{tr}}

\def\R{\mathbb{R}}

\def\cA{\mathcal{A}}

\def\cN{\mathcal{N}}
\def\cP{\mathcal{P}}

\def\cS{\mathcal{S}}

\def\cX{\mathcal{X}}
\def\cY{\mathcal{Y}}

\title{Revisiting differentially private linear regression: optimal and adaptive prediction \& estimation in unbounded domain}	

\author{
	Yu-Xiang Wang\thanks{Corresponding email: \texttt{yuxiangw@cs.ucsb.edu} } \\
		University of California, Santa Barbara, CA\\
		Amazon AWS AI, Palo Alto, CA\\
	\\
}
\date{}
\begin{document}
	
	\maketitle

	\begin{abstract}
	We revisit the problem of linear regression under a differential privacy constraint. By consolidating existing pieces in the literature, we clarify the correct dependence of the feature, label and coefficient domains in the optimization error and estimation error, hence revealing the delicate price of differential privacy in statistical estimation and statistical learning. Moreover, we propose simple modifications of two existing DP algorithms: (a) posterior sampling, (b) sufficient statistics perturbation, and show that they can be \emph{upgraded} into adaptive algorithms that are able to exploit data-dependent quantities and behave nearly optimally \emph{for every instance}.  Extensive experiments are conducted on both simulated data and real data, which conclude that both \AdaOPS{} and \AdaSP{} outperform the existing techniques on nearly all 36 data sets that we test on.  
	\end{abstract}
	
	\newpage
		\tableofcontents
\newpage
	
\section{Introduction}
Linear regression is one of the oldest tools for data analysis  \citep{galton1886regression} and it remains one of the most commonly-used as of today \citep{draper2014applied}, especially in social sciences \citep{agresti1997}, econometics \citep{greene2003econometric} and medical research \citep{armitage2008statistical}. Moreover, many nonlinear models are either intrinsically linear in certain function spaces, e.g., kernels methods, dynamical systems, or can be reduced to solving a sequence of linear regressions, e.g., iterative reweighted least square for generalized Linear models, gradient boosting for additive models and so on \citep[see][for a detailed review]{friedman2001elements}.  

In order to apply linear regression to sensitive data such as those in social sciences and medical studies, it is often needed to do so such that the privacy of individuals in the data set is protected.  Differential privacy \citep{dwork2006calibrating} is a commonly-accepted criterion that provides provable protection against identification and is resilient to arbitrary auxiliary information that might be available to attackers. In this paper, we focus on linear regression with $(\epsilon,\delta)$-differentially privacy \citep{dwork2006our}.

\paragraph{Isn't it a solved problem?}
It might be a bit surprising why this is still a problem, since several general frameworks of differential privacy have been proposed that cover linear regression.  Specifically, in the agnostic setting (without a data model), linear regression is a special case of differentially private empirical risk minimization (ERM), and its theoretical properties have been quite well-understood in a sense that the minimax lower bounds are known \citep{bassily2014private}  and a number of algorithms \citep{chaudhuri2011differentially,kifer2012private} have been shown to match the lower bounds under various assumptions. In the statistical estimation setting where we assume the data is generated from a linear Gaussian model, linear regression is covered by the sufficient statistics perturbation approach for exponential family models \citep{dwork2010differential,foulds2016theory}, propose-test-release framework \citep{dwork2009differential} as well as the the subsample-and-aggregate framework \citep{smith2008efficient}, with all three approaches achieving the asymptotic efficiency in the fixed dimension ($d=O(1)$), large sample  ($n\rightarrow \infty$) regime. 

Despite these theoretical advances, very few empirical evaluations of these algorithms were conducted and we are not aware of a commonly-accepted best practice. Practitioners are often left puzzled about which algorithm to use for the specific data set they have.  
The nature of differential privacy often requires them to set parameters of the algorithm (e.g., how much noise to add) according to the diameter of the parameter domain, as well as properties of a hypothetical worst-case data set, which often leads to an inefficient use of their valuable data.


The main contribution of this paper is threefold:
\begin{enumerate}
	\item We consolidated many bits and pieces from  the literature and clarified the price of differentially privacy in statistical estimation and statistical learning.
	\item We carefully analyzed One Posterior Sample  (\OPS{}) and Sufficient Statistics Perturbation (\SuffP{}) for linear regression and proposed simple modifications of them into adaptive versions: \AdaOPS{} and \AdaSP{}. Both work near optimally for every problem instance without any hyperparameter tuning.
	\item We conducted extensive real data experiments to benchmark existing techniques and concluded that the proposed techniques give rise to the more favorable privacy-utility tradeoff relative to existing methods.
\end{enumerate}

\paragraph{Outline of this paper.}
In Section~\ref{sec:setup} we will describe the problem setup and explain differential privacy. In Section~\ref{sec:priorwork}, we will survey the literature and discuss existing algorithms. 
Then we will propose and analyze our new method \AdaSP{} and \AdaOPS{} in Section~\ref{sec:main} and conclude the paper with experiments in Section~\ref{sec:exp}.

\section{Notations and setup}\label{sec:setup}
Throughout the paper we will use $X\in\R^{n\times d}$ and $\vct y\in \R^n$ to denote the design matrix and response vector. These are collections of data points $(x_1,y_1),...,(x_n,y_n)\in  \cX\times \cY$. 
We use $\|\cdot\|$ to denote Euclidean norm for vector inputs, $\ell_2$-operator norm for matrix inputs.  In addition, for set inputs, $\|\cdot\|$ denotes the radius of the smallest Euclidean ball that contains the set.  For example, $\|\cY\|  =  \sup_{y\in \cY} |y|$  and $\|\cX\| =  \sup_{x\in\cX}\|x\|$.  Let $\Theta$ be the domain of coefficients. Our results do not require $\Theta$ to be compact but  existing approaches often depend on $\|\Theta\|$.
$\lesssim$ and $\gtrsim$ denote greater than or smaller to up to a universal multiplicative constant, which is the same as the big $O(\cdot)$  and the big $\Omega(\cdot)$. $\tilde{O}(\cdot)$ hides at most a logarithmic term. $\prec$ and $\succ$ denote the standard semidefinite ordering of positive semi-definite (psd) matrices. $\cdot\vee\cdot$ and $\cdot \wedge \cdot$ denote the bigger or smaller of the two inputs.


We now define a few data dependent quantities. We use $\lambda_{\min }(X^TX)$ (abbv. $\lambda_{\min }$) to denote the smallest  eigenvalue of $X^TX$, and to make the implicit dependence in $d$ and $n$ clear from this quantity, we define 
$\alpha := \lambda_{\min }\frac{d}{n\|\cX\|^2}.$
One can think of $\alpha$ as a normalized smallest eigenvalue of $X^TX$ such that $0 \leq \alpha\leq 1$. Also, $1/\alpha$ is closely related to the condition number of $X^TX$. 

Define the least square solution
$\theta^*  = (X^TX)^{\dagger}X^T\vct y$. It is the optimal solution to 
$
\min_{\theta} \frac{1}{2}\|\vct y-X \theta\|^2 =: F(\theta).
$
Similarly, we use $\theta^*_\lambda = (X^TX + \lambda I)^{-1}X^T\vct y$ denotes the optimal solution to the ridge regression objective $F_\lambda(\theta)= F(\theta) + \lambda \|\theta\|^2$. 

In addition, we denote the global Lipschitz constant of $F$ as $L^*: = \|\cX\|^2\|\Theta\|+\|\cX\|\|\cY\|$ and data-dependent local Lipschitz constant at $\theta^*$ as $L := \|\cX\|^2\|\theta^*\|+\|\cX\|\|\cY\|$. Note that when $\Theta  =\R^d$, $L^* = \infty$, but $L$ will remain finite for every given data set.

\paragraph{Metric of success.}
We measure the performance of an estimator $\hat{\theta}$ in two ways. 

First, we consider the optimization error $F(\hat{\theta}) -  F(\theta^*)$ in expectation or with probability $1-\varrho$. This is related to the prediction accuracy in the distribution-free statistical learning setting.

Second, we consider how well the coefficients can be estimated under the linear Gaussian model:
$$
\vct y =  X\theta_0  + \cN(0,\sigma^2 I_n)
$$
in terms of $\E[\|\hat{\theta}-\theta_0\|^2]$ or in some cases $\E[ \|\hat{\theta}-\theta_0\|^2 |  E ] $ where $E$ is a high probability event.

The optimal error in either case will depend on the specific design matrix $X$, optimal solution $\theta^*$, the data domain $\cX,\cY$, the parameter domain $\Theta$ as well as $\theta_0,\sigma^2$ in the statistical estimation setting.  

\paragraph{Differential privacy.}
We will focus on estimators that are differential private, as defined below.
\begin{definition}[Differential privacy \citep{dwork2006calibrating}]
	We say a randomized algorithm $\cA$ satisfies $(\epsilon,\delta)$-DP if for \emph{all} fixed data set $(X,\vct y)$ and data set $(X',\vct y')$  that can be constructed by adding or removing one row $(x,y)$ from $(X,\vct y)$, and for any measurable set $\cS$ over the probability of the algorithm
	$$
	\P(\cA((X,\vct y))\in \cS)  \leq e^{\epsilon} \P(\cA((X',\vct y'))\in \cS) + \delta,\;\;
	$$                                                  
\end{definition}
Parameter $\epsilon$ represents the amount of privacy loss from running the algorithm and $\delta$ denotes a small probability of failure. These are user-specified targets to achieve and the differential privacy guarantee is considered meaningful if $\epsilon \leq 1$ and $\delta \ll 1/n$  \citep[see, e.g., Section 2.3.3 of ][for a comprehensive review]{dwork2014algorithmic}. 

\paragraph{The pursuit for adaptive estimators.}
Another important design feature that we will mention repeatedly in this paper is \emph{adaptivity}. We call an estimator $\hat{\theta}$ \emph{adaptive} if it behaves optimally simultaneously  for a wide range of parameter choices. Being adaptive is of great practical relevance because we do not need to specify the class of problems or worry about whether our specification is wrong \citep[see examples of adaptive estimators in e.g.,][]{donoho1995noising,birge2001gaussian}.  \emph{Adaptivity} is particularly important for differentially private data analysis because often we need to decide the amount of noise to add by the size of the domain.  For example, an adaptive algorithm will not rely on conservative upper bounds of $\theta_0$, or a worst case $\lambda_{\min }$ (which would be $0$ on any $\cX$), and it can take advantage of favorable properties when they exist in the data set.  We want to design an estimator that does not take these parameters as inputs and behave nearly optimally for every fixed data set $X\in \cX^n,\vct y\in \cY$  under a variety of configuration of $\|\cX\|,\|\cY\|,\|\Theta\|$.

\section{A survey of prior work}\label{sec:priorwork}
In this section, we summarize existing theoretical results in linear regression with and without differential privacy constraints. We will start with lower bounds.

\subsection{Information-theoretic lower bounds}

\paragraph{Lower bounds under linear Gaussian model.}
Under the statistical assumption of linear Gaussian model $\vct y  =  X\theta_0 + \cN(0,\sigma^2)$, the minimax risk for both estimation and prediction are crisply characterized for each fixed design matrix $X$:
\begin{equation}\label{eq:prediction_lowerbound}
\inf_{\hat{\theta}}\sup_{\theta_0\in\R^d} \E[ F(\hat{\theta}) - F(\theta_0)| X] 
=  \frac{d\sigma^2}{2},
\end{equation}	
and	if we further assume that $n\geq d$ and $X^TX$ is invertible (for identifiability), then
\begin{align}\label{eq: estimation_lowerbound}
\inf_{\hat{\theta}}\sup_{\theta_0\in\R^d} \E[\|\hat{\theta}-\theta_0\|^2_2 | X]  =  \sigma^2\tr[(X^TX)^{-1}] .
\end{align}
In the above setup, $\hat{\theta}$ is any measurable function of $\hat{y}$ (note that $X$ is fixed). These are classic results that can be found in standard statistical decision theory textbooks \citep[See, e.g., ][Chapter 13]{wasserman2013all}.

Under the same assumptions, the Cramer-Rao lower bound mandates that the covariance matrix of any unbiased estimator $\hat{\theta}$ of $\theta_0$ to obey that
\begin{equation}\label{eq:cramer-rao}
\Cov(\hat{\theta})  \succ  \sigma^2 (X^TX)^{-1}.
\end{equation}
This bound applies to every problem instance separately and also implies a sharp lower bound on the prediction variance on every data point $x$. More precisely, $\Var(\hat{\theta}^Tx) \geq \sigma^2 x^T(X^TX)^{-1}x$ for any $x$. 

Minimax risk \eqref{eq:prediction_lowerbound}, \eqref{eq: estimation_lowerbound} and the Cramer-Rao lower bound \eqref{eq:cramer-rao} are simultaneously attained by $\theta^*$.


\paragraph{Statistical learning lower bounds.}
Perhaps much less well-known, linear regression is also thoroughly studied in the distribution-free statistical learning setting, where the only assumption is that the data are drawn iid from some unknown distribution $\cP$ defined on some compact domain $\cX\times \cY$. Specifically, let the risk ($\E[\text{loss}]$) be 
$$
R(\theta) = \E_{( x,y)\sim \cP}[ {\textstyle \frac{1}{2}}(x^T\theta-y)^2]  ={\textstyle \frac{1}{n}}\E_{(X,\vct y)\sim \cP^n}[  F(\theta)].
$$
\citet{shamir2015sample} showed that when $\Theta$, $\cX$ are $\cY$ are Euclidean balls,
\begin{equation}\label{eq:shamir_bound}
\begin{aligned}
&\inf_{\hat{\theta}}\sup_{\cP} \left[\E [n\cdot R(\hat{\theta})] - \inf_{\theta\in\Theta}[n\cdot R(\theta)]  \right]\\
\gtrsim& \min\{ n\|\cY\|^2, \|\Theta\|^2\|\cX\|^2 + d \|\cY\|^2, \sqrt{n}  \|\Theta\|\|\cX\|\|\cY\| \}.
\end{aligned}
\end{equation}
where $\hat{\theta}$ be any measurable function of the data set $X,\vct y$ to $\Theta$ and the expectation is taken over the data generating distribution $X,\vct y\sim \cP^n$. Note that to be compatible to other bounds that appear in this paper, we multiplied the $R(\cdot)$ by a factor of $n$. Informally, one can think of $\|\cY\|$ as $\sigma$ in \eqref{eq:prediction_lowerbound} so both terms depend on $d\sigma^2$ (or $d\|\cY\|^2$), but the dependence on $\|\Theta\|\|\cX\|$ is new for the distribution-free setting. 

\citet{koren2015fast} later showed that this lower bound is matched up to a constant by Ridge Regression with $\lambda = 1$
and both \citet{koren2015fast} and \citet{shamir2015sample} conjecture that ERM without additional regularization should attain the lower bound \eqref{eq:shamir_bound}. If the conjecture is true, then the unconstrained OLS is simultaneously optimal for all distributions supported on the smallest ball that contains all data points in $X,\vct y$ for any $\Theta$ being an $\ell_2$ ball with radius larger than $\|\theta^*\|$.


\paragraph{Lower bounds with $(\epsilon,\delta)$-privacy constraints.}
Suppose that we further require $\hat{\theta}$ to be $(\epsilon,\delta)$-differentially private, then there is an additional price to pay in terms of how accurately we can approximate the ERM solution. Specifically, the lower bounds for the \emph{empirical} excess risk for differentially private ERM problem in \citep{bassily2014private} implies that for $\delta<1/n$ and sufficiently large $n$:
\begin{itemize}
	\item[1.] There exists a triplet of $(\cX,\cY, \Theta)\subset \R^d\times \R\times \R^d$,  
	such that
	\begin{equation}\label{eq:dp_lowerbound_lipschitz}
	\begin{aligned}
	\inf_{\hat{\theta}\text{ is }(\epsilon,\delta)\text{-DP}}\sup_{X\in \cX^n,\vct y\in \cY^n}  \left[F(\hat{\theta}) - \inf_{\theta\in\Theta}F(\theta)
	\right] 
	\gtrsim \min\left\{n\|\cY\|^2, \frac{\sqrt{d} (\|\cX\|^2\|\Theta\|^2 + \|\cX\|\|\Theta\|\|\cY\|)}{\epsilon}\right\}.
	\end{aligned}
	\end{equation}
	\item[2.]  Consider the class of data set $\cS$ where all data sets $X\in \cS \subset \cX^n$ obeys that the inverse condition number $\alpha \geq \alpha^* \geq \frac{d^{1.5}(\|\cX\|\|\Theta\| + \|\cY\|)}{n\|\cX\|\|\Theta\|\epsilon}$ \footnote{This requires $\lambda_{\min } \geq \sqrt{d}L/\epsilon$ for all data sets $X$.}. 
	There exists a triplet of $(\cX,\cY, \Theta)\subset \R^d\times \R\times \R^d$ such that 
	\begin{equation}\label{eq:dp_lowerbound_strongcvx}
	\begin{aligned}
	\inf_{\hat{\theta}\text{ is }(\epsilon,\delta)\text{-DP}}\sup_{X\in \cS,\vct y\in \cY^n} \left[F(\hat{\theta}) - \inf_{\theta\in\Theta}F(\theta)
	\right] 
	\gtrsim  \min\left\{n\|\cY\|^2, \frac{d^2 (\|\cX\|\|\Theta\| + \|\cY\|)^2 }{n\alpha^* \epsilon^2}\right\}.
	\end{aligned}
	\end{equation}
\end{itemize}
These bounds are attained by a number of algorithms, which we will go over in Section~\ref{sec:exist_algs}.


Comparing to the non-private minimax rates on prediction accuracy, the bounds look different in several aspects. First, neither rate for prediction error in \eqref{eq:prediction_lowerbound} or \eqref{eq:shamir_bound} depends on whether the design matrix $X$ is well-conditioned or not, while $\alpha^*$ appears explicitly in \eqref{eq:dp_lowerbound_strongcvx}. 
Secondly, the dependence on $\|\Theta\|\|\cX\|,\|\cY\|,d,n$ are different, which makes it hard to tell whether the optimization error lower bound due to the privacy requirement is limiting.


To clarify the relationships, we plot Shamir's lower bound \eqref{eq:shamir_bound}  and the smaller of Bassily et. al.'s differential privacy lower bounds  \eqref{eq:dp_lowerbound_lipschitz} and \eqref{eq:dp_lowerbound_strongcvx} for all configurations of $d,n$ graphically in Figure~\ref{fig:price_of_DP}. We also use multiple lines to illustrate the shifts in these lower bounds when parameters such as $\epsilon$ and $\alpha^*$ changes. In all figures $\delta$ is assumed to be $o(1/n)$ and logarithmic terms are dropped.  The price of differential privacy is highlighted as a shaded area in the figures. Interestingly, in the first case when $\|\Theta\|$ is small (when $\|\cX\|\|\Theta\|\asymp \|\cY\|$), then substantial price only occurs in the non-standard region where $n<d$.  Arguably this is OK because in that regime, people should use Ridge regression or Lasso anyways rather than OLS. In the case when $\|\Theta\|$ is large (when $\|\cX\|\|\Theta\|\asymp d\|\cY\|$), the price is more substantial and it applies to all $n>d$ unless we can exploit the strong convexity in the data set. When we do, then the cost only occur for an interval in $n$ and eventually the cost of differential privacy becomes negligible relative to the minimax rate.  
To the best of our knowledge this is the first time the ``price of differential privacy'' for linear regression is discussed with clear explanation of the dependency in all parameters of the problem. 

\begin{figure}	[p]
	\centering
	\begin{subfigure}[t]{0.48\textwidth}
		\centering
		\includegraphics[width=\textwidth]{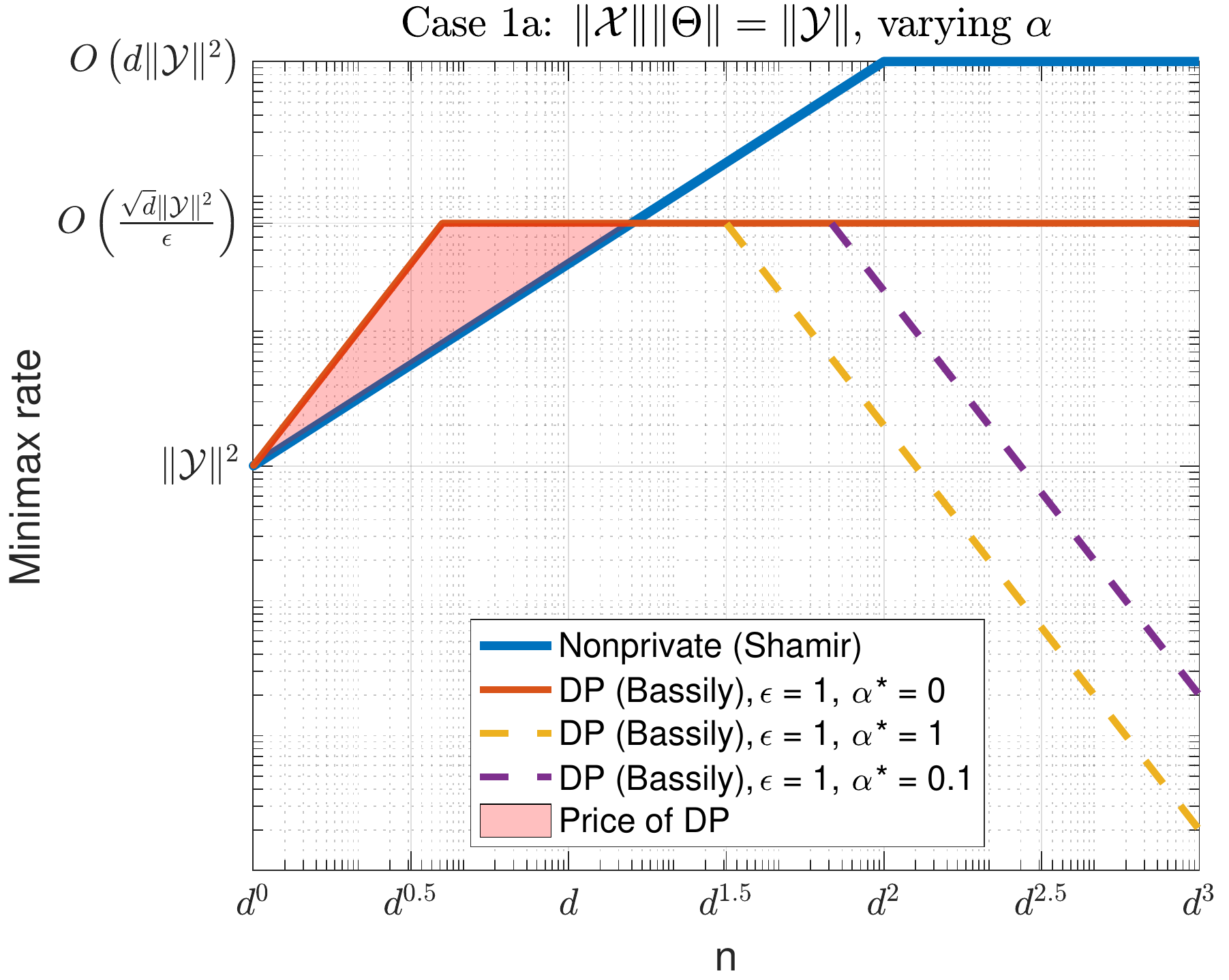}
	\end{subfigure}
	\begin{subfigure}[t]{0.48\textwidth}
		\centering
		\includegraphics[width=\textwidth]{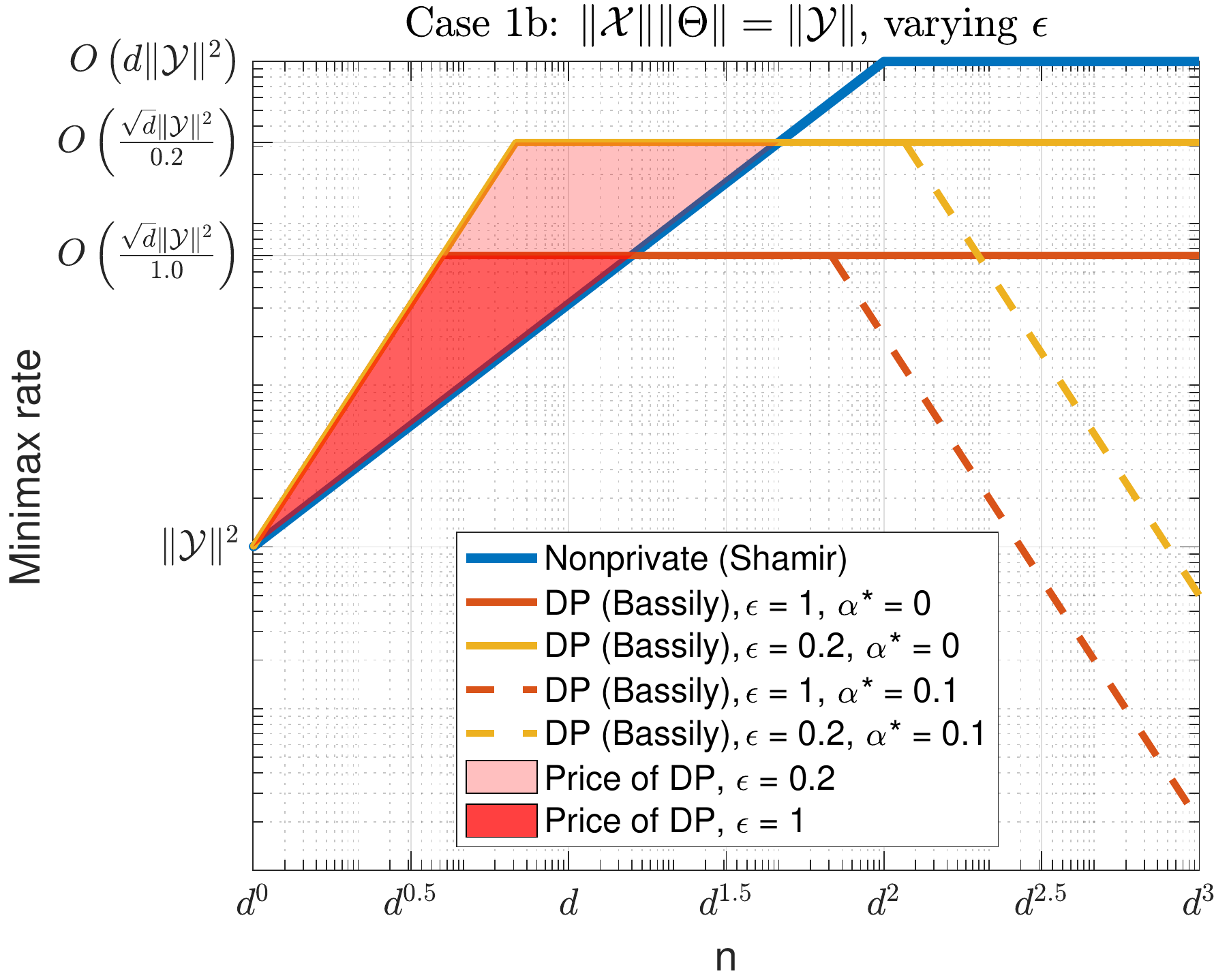}
	\end{subfigure}
	\begin{subfigure}[t]{0.48\textwidth}
		\centering
		\includegraphics[width=\textwidth]{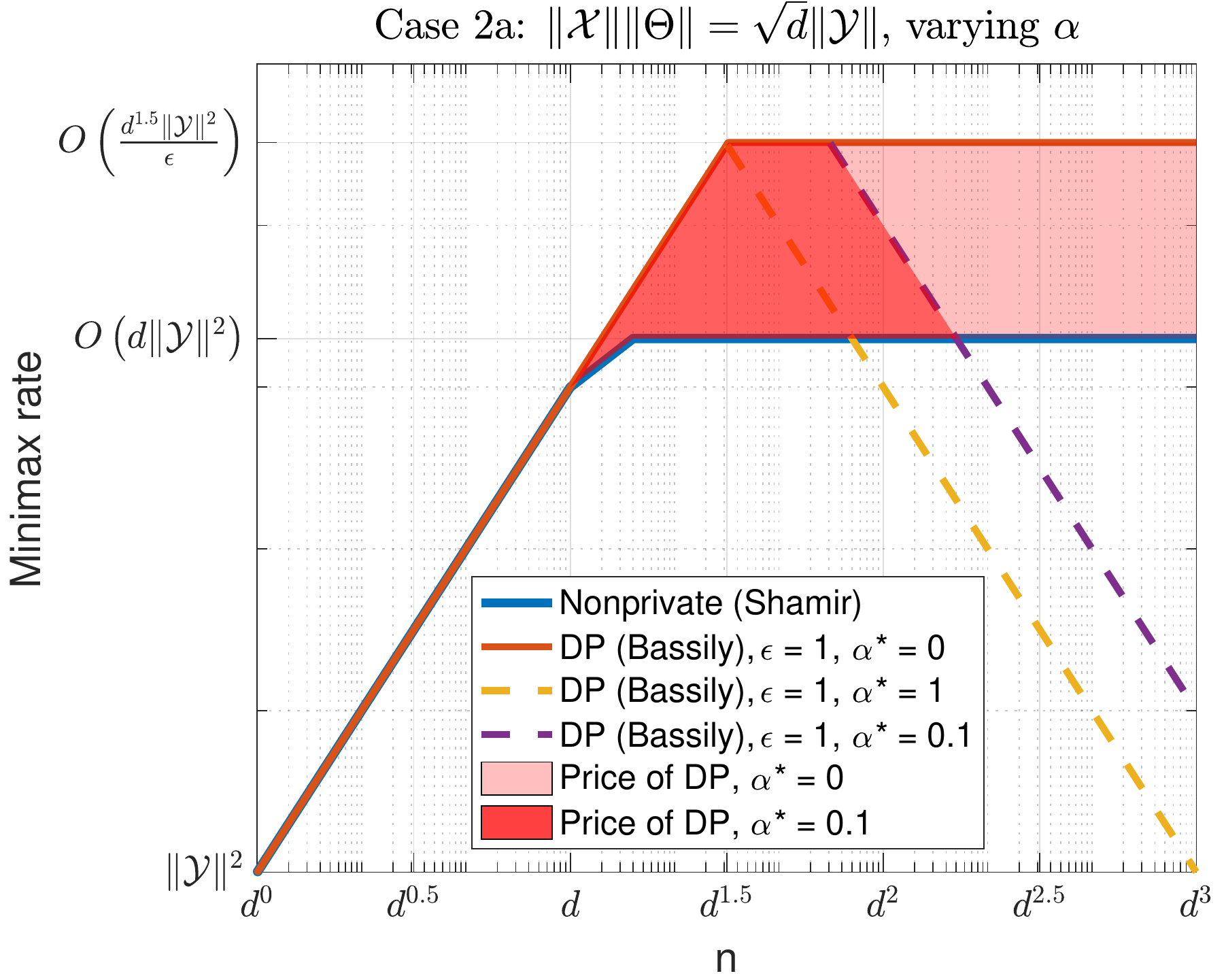}
	\end{subfigure}
	\begin{subfigure}[t]{0.48\textwidth}
		\centering
		\includegraphics[width=\textwidth]{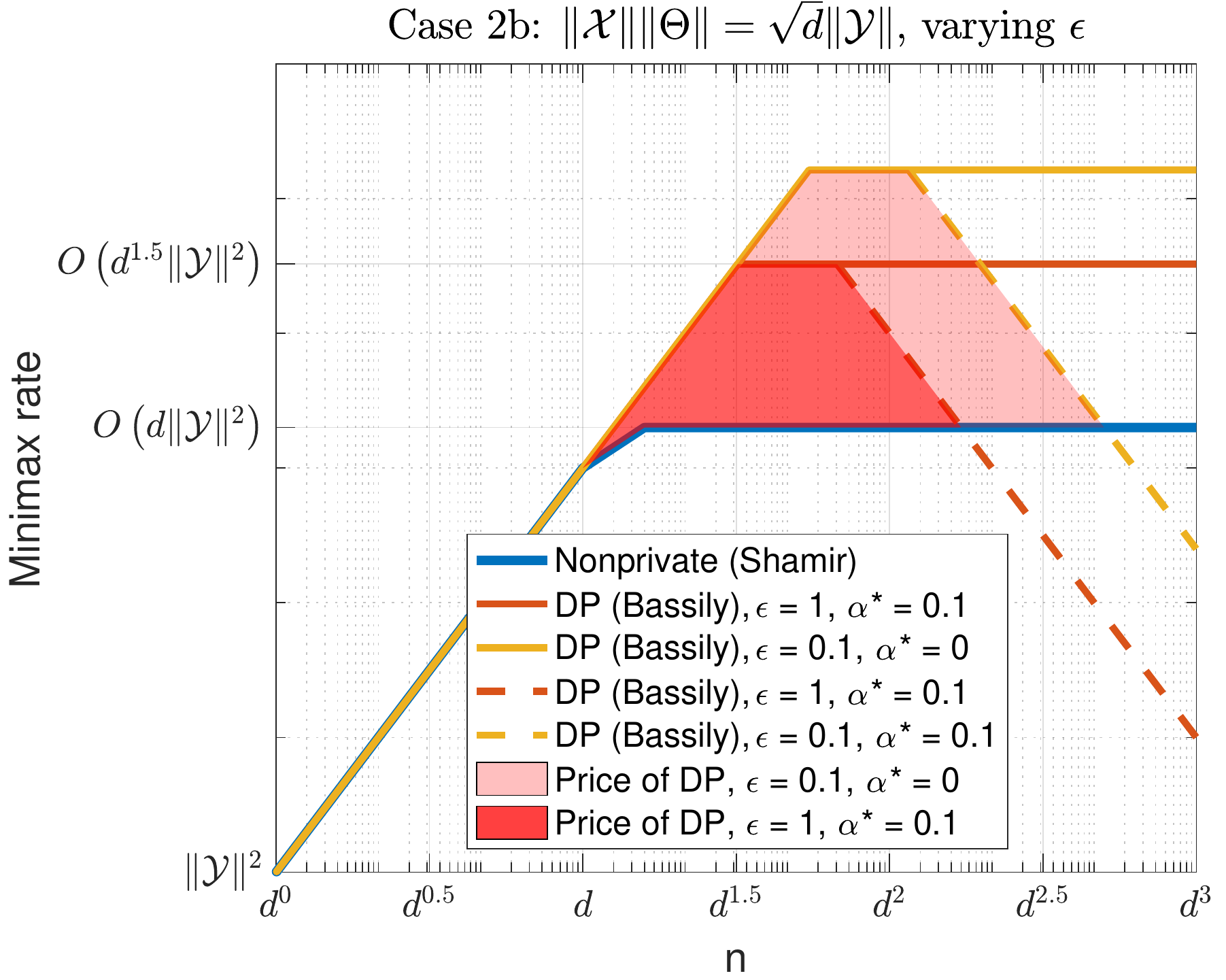}
	\end{subfigure}
	\caption{Illustration of the lower bounds for non-private and private linear regression.}\label{fig:price_of_DP}
	\bigskip
	\centering
	\includegraphics[width=0.44\textwidth]{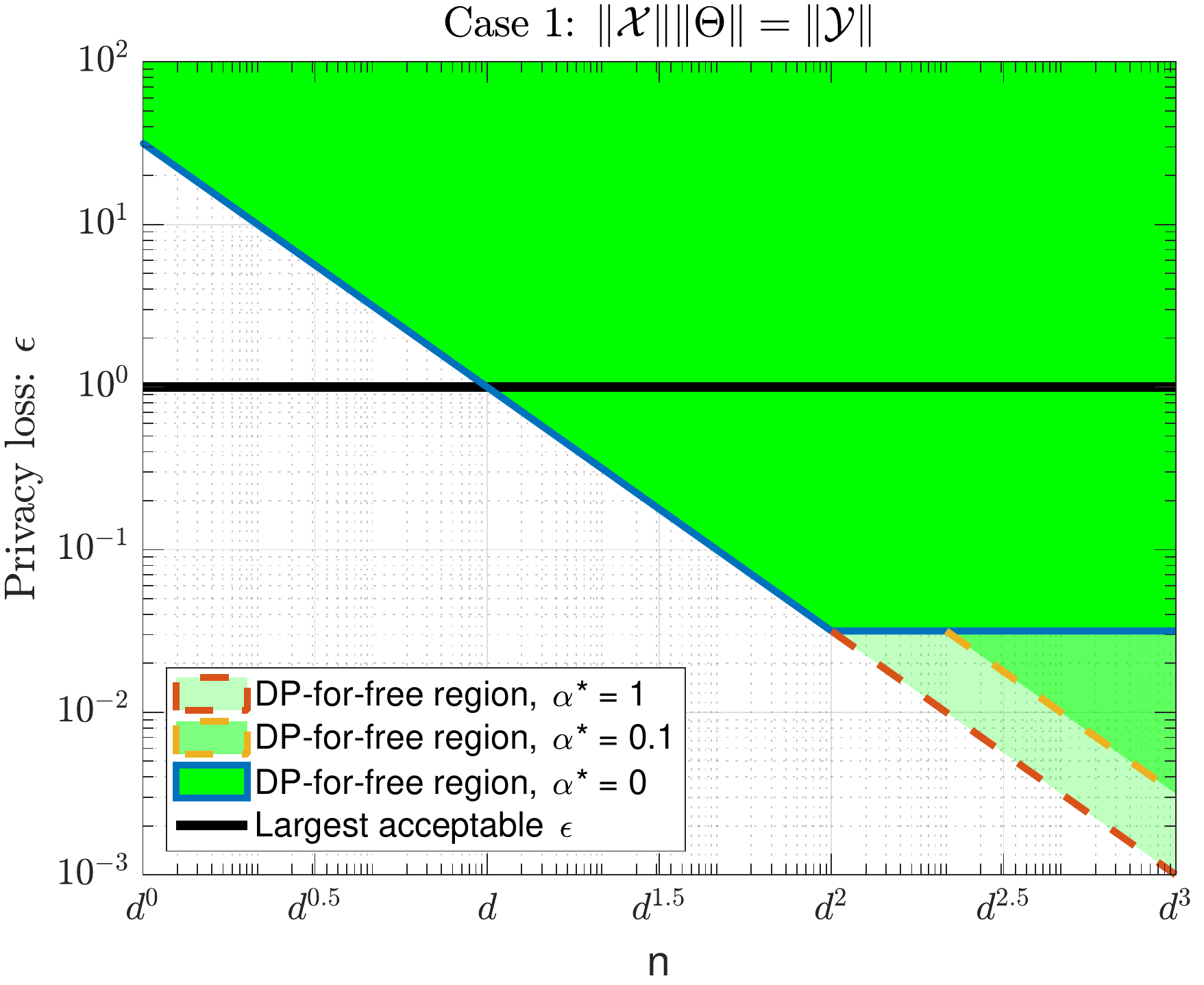}
	\includegraphics[width=0.44\textwidth]{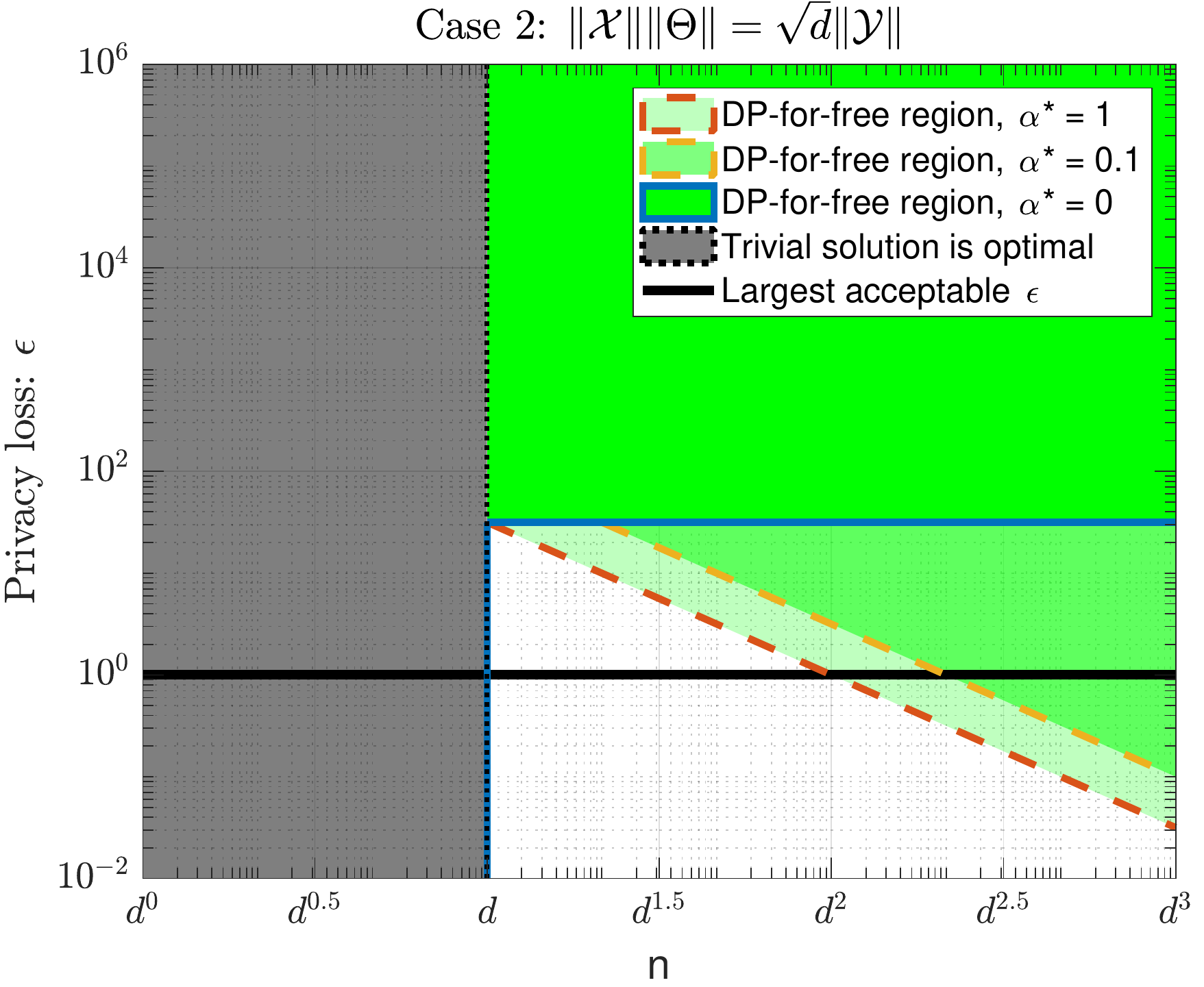}
	\caption{Illustration of the region of $\epsilon$ where DP can be obtained without losing the statistical learning minimax rate.}\label{fig:regions_PfF}
\end{figure}

The above discussion also allows us to address the following question.
\begin{center}
	When is privacy \emph{for free} in statistical learning?
\end{center}
Specifically, what is the smallest  $\epsilon$ such that an $(\epsilon,\delta)$-DP algorithm matches the minimax rate in \eqref{eq:shamir_bound}?
The answer really depends on the relative scale of  $\|\cX\|\|\Theta\|$ and $\|\cY\|$ and that of $n,d$.  When $\|\cX\|\|\Theta\|\asymp \|\cY\|$, \eqref{eq:dp_lowerbound_lipschitz} says that $(\epsilon,\delta)$-DP algorithms can achieve the nonconvex minimax rate provided that 
$\epsilon \gtrsim  \min\left\{ \frac{1}{\sqrt{d}}\vee\sqrt{\frac{d}{n}}, \sqrt{\frac{d^2}{n^{1.5}\alpha^*}} \vee \sqrt{\frac{d}{n\alpha^*}}\right\}. $
On the other hand, if  $\|\cX\|\|\Theta\|\asymp \sqrt{d}\|\cY\|$ 
\footnote{This is arguably the more relevant setting. Note that if $x\sim \cN(0,I_d)$  and $\theta$ is fixed, then $x^T\theta = O_P(d^{-1/2}\|x\|\|\theta\|)$.} 
and $n>d$, then we need 
$\epsilon \gtrsim \min\left\{\sqrt{d}\vee\frac{d^{3/2}}{n},  \frac{d}{\sqrt{n\alpha^*}}\vee \frac{d^{3/2}}{n\sqrt{\alpha^*}} \right\}.$ 

The regions are illustrated graphically in Figure~\ref{fig:regions_PfF}. 
In the first case, there is a large region upon $n\gtrsim d$, where meaningful differential privacy (with $\epsilon\leq 1$ and $\delta= o(1/n)$) can be achieved without incurring a significant toll relative to \eqref{eq:shamir_bound}. In the second case, we need at least $n\gtrsim d^2$ to achieve ``privacy-for-free'' in the most favorable case where $\alpha^*=1$. In the case when $X$ could be rank-deficient, then it is infeasible to achieve ``privacy for free'' no matter how large $n$ is. 

Based on the results in Figure~\ref{fig:price_of_DP} and ~\ref{fig:regions_PfF}, it might be tempting to conclude that one should always prefer  Case 1 over Case 2. This is unfortunately not true because the artificial restriction of the model class via a bounded $\|\Theta\|$ also weakens our non-private baseline. In other word, the best solution within a small $\Theta$ might be significantly worse than the best solution in $\R^d$.

In practice, it is hard to  find a $\Theta$ with a small radius that fits all purposes\footnote{If $\|\Theta\| \gg \|\theta^*\|$ then the constraint becomes limiting. If $\|\theta^*\| \ll \|\Theta\|$ instead, then calibrating the noise according to $\|\Theta\|$ will inject more noise than necessary.} and it is unreasonable to assume $\alpha^*>0$. This motivates us to go beyond the worst-case and come up with \emph{adaptive} algorithms that work without knowing $\|\theta^*\|$ and $\alpha$ while achieving the minimax rate for the class with $\|\Theta\| = \|\theta^*\|$ and $\alpha^*=\alpha$ (in hindsight).

\subsection{Existing algorithms and our contribution }\label{sec:exist_algs}
We now survey the following list of five popular algorithms in differentially private learning and highlight the novelty in our proposals \footnote{While we try to be as comprehensive as possible,  the literature has grown massively and the choice of this list is limited by our knowledge and opinions.}.
\begin{enumerate}
	\item Sufficient statistics perturbation (\SuffP{})   \citep{vu2009differential,foulds2016theory}:  Release $X^TX$ and $X\vct y$ differential privately and then output $\hat{\theta} = (\widehat{X^TX})^{-1}\widehat{X\vct y}$.
	\item Objective perturbation (\OBP{}) \citep{kifer2012private}:  $\hat{\theta} =  \argmin  F(\theta)  +  0.5\lambda \|\theta\|^2  +  Z^T\theta$ with an appropriate $\lambda$ and $Z$ is an appropriately chosen iid Gaussian random vector.
	\item Subsample and Aggregate (Sub-Agg) \citep{smith2008efficient,dwork2010differential}:  Subsample many times, apply debiased MLE to each subset and then randomize the way we aggregate the results.
	\item Posterior sampling (\OPS{}) \citep{mir13,dimitrakakis2014robust,wang2015privacy,minami2016differential}:   Output $\hat{\theta} \sim P(\theta) \propto e^{  -  \gamma (F(\theta) + 0.5\lambda \|\theta\|^2)}$ with parameters $ \gamma,\lambda$.
	\item \NSGD{} \citep{bassily2014private}:  Run SGD for a fixed number of iterations with additional Gaussian noise added to the stochastic gradient evaluated on one randomly-chosen data point.
\end{enumerate}
We omit detailed operational aspects of these algorithms and focus our discussion on their theoretical guarantees. Interested readers are encouraged to check out each paper separately.
These algorithms are proven under different scalings and assumptions. To ensure fair comparison, we make sure that all results are converted to our setting under a subset of the following assumptions.

\begin{table}[t]
	\bigskip
	\bigskip
	\caption{Summary of optimization error bounds. This table compares the (expected or high probability ) additive suboptimality of different differentially private linear regression procedures relative to the (non-private) empirical risk minimizer $\theta^*$. In particular, the results for NoisySGD holds in expectation and everything else with probability $1-\varrho$ (hiding at most a logarithmic factor in $\sqrt{1/\varrho}$). Constant factors are dropped for readability. 
	}\label{tab:comparison}
	\centering
	\resizebox{\textwidth}{!}{
		\begin{tabular}{cccp{5cm}}
			\midrule
			&  $F(\hat{\theta}) - F(\theta^*)$  & Assumptions  & Remarks\smallskip\\ 
			\midrule
			\multirow{2}{*}{NoisySGD} & $\frac{\sqrt{d\log(\frac{n}{\delta})}\|\cX\|^2\|\Theta\|^2}{\epsilon}$    & A.1, A.2  & Theorem 2.4 (Part 1) of \citep{bassily2014private}.\smallskip\\
			&$\frac{d^2 \log(\frac{n}{\delta})\|\Theta\|^2}{\alpha^* n\epsilon^2}$    &  A.1, A.2, A.3 & Theorem 2.4 (Part 2) of \citep{bassily2014private} \smallskip\\
			\midrule
			\multirow{2}{*}{\OBP{}}  & $\frac{\sqrt{d\log(\frac{1}{\delta})}\|\cX\|^2\|\Theta\|\|\theta^*\|}{\epsilon}$ & A.1, A.2 &  Theorem 4 (Part 2) of \citep{kifer2012private}. \smallskip\\
			& $\frac{d^2 \log(\frac{1}{\delta})\|\Theta\|^2}{\alpha^* n\epsilon^2}$ & A.1, A.2, A.3 &  Theorem 5 \& Appendix E.2 of \citep{kifer2012private}. \\
			\midrule
			\OPS{} &  $\frac{d  \|\cX\|^2\|\Theta\|^2}{\epsilon}$ & A.1, A.2&  Results for $\epsilon$-DP \citep{wang2015privacy}\smallskip\\
			\midrule
			\SuffP{} & $\frac{d^2\log(\frac{1}{\delta})\|\cX\|^2\|\theta^*\|^2}{\alpha n\epsilon^2}$	&   A.1 & Adaptive to $\|\theta^*\|,X,\alpha$, but requires  $n =\Omega(\frac{d^{1.5}\log(4/\delta)}{\alpha \epsilon})$ \footnotemark.\smallskip\\
			\midrule
			\AdaOPS{} \& \AdaSP{}&  $\frac{\sqrt{d\log(\frac{1}{\delta})}\|\cX\|^2\|\theta^*\|^2}{\epsilon}\wedge \frac{d^2 \log(\frac{1}{\delta})\|\theta^*\|^2}{\alpha n\epsilon^2}$ & A.1 & Adaptive in $\|\theta^*\|,X,\alpha$.\smallskip\\
			\midrule
		\end{tabular}
	}
	
\end{table}
\begin{table}[t]
	\caption[Summary of estimation errors under the linear Gaussian model]{Summary or estimation error bounds under the linear Gaussian model. On the second column we compare the approximation of MLE $\theta^*$ in mean square error up to a universal constant.
		On the third column, we compare the relative efficiency. The relative efficiency bounds are simplified with the assumption of $\alpha = \Omega(1)$, which implies that $\tr[(X^TX)^{-1}] =  O(d^2n^{-1}\|\cX\|^{-2})$ and $\tr[(X^TX)^{-2}] =  O(dn^{-1}\|\cX\|^{-2} \tr[(X^TX)^{-1}])$. $\tilde{O}(\cdot)$ hides
		$\mathrm{polylog}(1/\delta)$ terms. }\label{tab:compare-efficiency}
	\resizebox{\textwidth}{!}{
		\begin{tabular}{cccp{3.5cm}}
			\midrule
			&Approxi. MLE: $\E\|\hat{\theta}-\theta^*\|^2$& Rel. efficiency: $\frac{\E \|\hat{\theta}-\theta_0\|^2}{\E\|\theta^* - \theta_0\|^2}$ &Remarks \smallskip\\
			\midrule
			Sub-Agg &  $ O\left(\frac{\text{poly}(d,\|\Theta\|,\|\cX\|,\alpha^{-1})}{\epsilon^{6/5}n^{6/5}}\right)$   &  $ 1 + \tilde{O}(\frac{\text{poly}(d,\|\Theta\|,\|\cX\|)}{n^{1/5}\epsilon^{6/5}})$ 	&$\epsilon$-DP, suboptimal in $n$, possibly also in $d$\citep{dwork2010differential}.	\smallskip\\
			\OPS{} & $O(\frac{\|\cX\|^2\|\Theta\|^2}{\epsilon})\tr[(X^TX)^{-1}]$ & $\tilde{O}(\frac{\|\cX\|^2\|\Theta\|^2}{\epsilon \sigma^2})$ & $\epsilon$-DP, adaptive in $X$, but not asymptotically efficient \citep{wang2015privacy}. \smallskip\\
			\SuffP{} &
			$O\left(  \frac{\log(\frac{1}{\delta}) \|\cX\|^4\|\theta^*\|^2}{\epsilon^2} \tr[(X^TX)^{-2}]  \right)$
			&  
			$1+ \tilde{O}(\frac{d\|\cX\|^2\|\theta_0\|^2}{n\epsilon^2\sigma^2}  +  \frac{d^3}{n^2\epsilon^2})$
			& Adaptive in $\|\theta^*\|,X$, no explicit dependence on $\alpha$, but requires large $n$.	 \citep[Theorem 5.1]{sheffet2017differentially} \smallskip\\
			\AdaOPS{} \& \AdaSP{}& $O\left(\frac{d\log(\frac{1}{\delta}) \|\cX\|^2\|\theta^*\|^2 }{\alpha n \epsilon^2}\tr[(X^TX)^{-1}] \right)$& 
			$1+\tilde{O}(\frac{d\|\cX\|^2\|\theta_0\|^2}{n\epsilon^2\sigma^2}+\frac{d^3}{n^2\epsilon^2})$ & Adaptive  in $\|\theta^*\|,X,\alpha$.\\
			\midrule
			
		\end{tabular}
	}
\end{table}

\begin{enumerate}
	\item[A.1] $\|\cX\|$ is bounded, $\|\cY\|$ is bounded. 
	\item[A.2] $\|\Theta\|$ is bounded.
	\item[A.3] All possible data sets $X$ obey that the smallest eigenvalue $\lambda_{\min }(X^TX)$ is greater than $\frac{n\|\cX\|^2}{d}\alpha^*$. 
\end{enumerate}
Note that A.3 is a restriction on the domain of the data set, rather than the domain of individual data points in the data set of size $n$. While it is a little unconventional, it is valid to define differential privacy within such a restricted space of data sets. It is the same assumption that we needed to assume for the lower bound in \eqref{eq:dp_lowerbound_strongcvx} to be meaningful.  As in \citet{koren2015fast}, we simplify the expressions of the bound by assuming $\|\cY\|\leq \|\cX\|\|\Theta\|$, and in addition, we assume that $\|\cY\| \lesssim \|\cX\|\|\theta^*\|$.

Table~\ref{tab:comparison} summarizes the upper bounds of optimization error the aforementioned algorithms in comparison to our two proposals: \AdaOPS{} and \AdaSP{}.  
Comparing the rates to the lower bounds in the previous section, it is clear that NoisySGD, \OBP{} both achieve the minimax rate in optimization error 
but their hyperparameter choice depends on the unknown $\|\Theta\|$ and $\alpha^*$. \SuffP{} is adaptive to $\alpha$ and $\|\theta^*\|$ but has a completely different type of issue --- it can fail arbitrarily badly for regime covered under \eqref{eq:dp_lowerbound_lipschitz}, and even for well-conditioned problems, its theoretical guarantees only kick in as $n$ gets very large. Our proposed algorithms \AdaOPS{} and \AdaSP{} are able to simultaneously switch between the two regimes and get the best of both worlds.

Table~\ref{tab:compare-efficiency} summarizes the upper bounds for estimation. The second row compares the approximation of $\theta^*$ in MSE and the third column summarizes the statistical efficiency of the DP estimators relative to the MLE: $\theta^*$ under the linear Gaussian model.
All algorithms except \OPS{} are asymptotically efficient. For the interest of $(\epsilon,\delta)$-DP, \SuffP{} has the fastest convergence rate and does not explicitly depend on the smallest eigenvalue, but again it behaves differently when $n$ is small, while \AdaOPS{} and  \AdaSP{} work optimally (up to a constant) for all $n$. 

\subsection{Other related work}
The problem of adaptive estimation is closely related to model selection \citep[see, e.g.,][]{birge2001gaussian}  and an approach using Bayesian Information Criteria was carefully studied  in the differential private setting for the problem of $\ell_1$ constrained ridge regression by \citet{lei2016differentially}.  Their focus is different to ours in that they care about inferring the correct model, while we take the distribution-free view.
Linear regression is also studied in many more specialized setups, e.g., high dimensional linear regression \citep{kifer2012private,talwar2014private,talwar2015nearly}, statistical inference \citep{sheffet2017differentially} and so on. For the interest of this paper, we focus on the standard regime of linear regression where $d<n$ and do not use sparsity or $\ell_1$ constraint set to achieve the $\log(d)$ dependence. That said, we acknowledge that \citet{sheffet2017differentially} analyzed \SuffP{} under the linear Gaussian model (the third row in Table~\ref{tab:compare-efficiency}and their techniques of adaptively adding regularization have inspired \AdaSP{}.

\section{Main results: adaptive private linear regression}	 \label{sec:main}	

In this section, we present and analyze \AdaOPS{} and \AdaSP{} that achieve the aforementioned adaptive rate. The pseudo-code of these two algorithms are given in Algorithm~\ref{alg:self-tuning-AdaOPS} and Algorithm~\ref{alg:adaSuffP}.

The idea of both algorithms is to release key data-dependent quantities differentially privately and then use a high probability confidence interval of these quantities to calibrate the noise to privacy budget as well as to choose the ridge regression's hyperparameter $\lambda$ for achieving the smallest prediction error.  Specifically, \AdaOPS{} requires us to release both the smallest eigenvalue $\lambda_{\min }$ of $X^TX$ and the local Lipschitz constant $L := \|\cX\|(\|\cX\|\|\theta^*_\lambda\| + \|\cY\|)$, while \AdaSP{} only needs the smallest eigenvalue $\lambda_{\min }$.


\begin{algorithm}[t]                   
	\caption{\AdaOPS{}: One-Posterior Sample estimator with adaptive regularization}          
	\label{alg:self-tuning-AdaOPS}                           
	\begin{algorithmic}                    
		\INPUT{ Data $X$, $\vct y$. Privacy budget: $\epsilon$, $\delta$, Bounds: $\|\cX\|, \|\cY\|$. 
		}
		\STATE{1. Calculate the minimum eigenvalue $\lambda_{\text{min}}(X^TX)$.}
		\STATE{2. Sample $Z\sim \cN(0,1)$ and privately release \\
			$\tilde{\lambda}_{\text{min}} =  \max\left\{\lambda_{\text{min}} + \frac{\sqrt{\log(6/\delta)}}{\epsilon/4}Z  - \frac{\log(6/\delta)}{\epsilon/4}, 0\right\}.$}
		\STATE{ 3. Set $\bar{\epsilon}$ as the positive solution of the quadratic equation
			$$
			\bar{\epsilon}^2/ (2\log(6/\delta))+ \bar{\epsilon}  -  \epsilon/4 = 0.
			$$}
		\STATE{4. Set $\varrho=0.05$,
			$
			C_1 = \big(d/2 + \sqrt{d\log(1/\varrho)} + \log(1/\varrho)\big)\log(6/\delta)/\bar{\epsilon}^2,
			$
			$
			C_2 =  \log(6/\delta)/(\epsilon/4),
			$
			$
			t_{\min } = \max\{\frac{\|\cX\|^2(1+\log(6/\delta))}{2\epsilon}-\tilde{\lambda}_{\min },0\} 
			$
			and solve 
			\begin{equation}\label{eq:adachoice_of_lambda}
			\lambda =  \argmin_{t \geq t_{\min }}\frac{\|\cX\|^4C_1[1+\|\cX\|^2/( t + \tilde{\lambda}_{\min })]^{2 C_2}}{ t + \tilde{\lambda}_{\min }  } + t.
			\end{equation}
			which has a unique solution.
		}
		\STATE{5. Calculate
			$
			\hat{\theta}  = (X^TX + \lambda I)^{-1}X^T\vct y.
			$}
		\STATE{6. Sample $Z\sim \cN(0,1)$ and privately release} \\
		$\Delta = \log(\|\cY\|+\|\cX\|\|\hat{\theta} \|) + \frac{\log(1+\|\cX\|^2/(\lambda+\tilde{\lambda}_{\min }))}{\epsilon/ (4\sqrt{\log(6/\delta)})} Z +  \frac{\log(1+\|\cX\|^2/(\lambda+\tilde{\lambda}_{\min }))}{\epsilon/(4\log(6/\delta))}$. Set $\tilde{L}  :=  \|\cX\|e^{\Delta}$.
		\STATE{7.  Calibrate noise by choosing  
			$\tilde{\epsilon}$ as the positive solution of the quadratic equation
			\begin{equation}\label{eq:adachoice_of_eps}
			\frac{\tilde{\epsilon}^2}{2} \left[\frac{1}{\log(6/\delta)}\frac{1+\log(6/\delta)}{\log(6/\delta)}\right] + \tilde{\epsilon}  - \epsilon/2  =  0. 
			\end{equation}
			and then set
			$\gamma = \frac{(\tilde{\lambda}_{\min } +\lambda)\tilde{\epsilon}^2}{\log(6/\delta) \tilde{L}^2}.$ }			
		\OUTPUT{ $\tilde{\theta}\sim p(\theta|X,\vct y) \propto e^{-\frac{\gamma}{2} \left(\|\vct y - X\theta\|^2 + \lambda\|\theta\|^2\right)}.$}
	\end{algorithmic}
\end{algorithm}

\begin{algorithm}[t]                        
	\caption{\AdaSP{}: Sufficient statistics perturbation with adaptive damping}
	\label{alg:adaSuffP}   
	\begin{algorithmic}                    
		\INPUT{ Data $X$, $\vct y$. Privacy budget: $\epsilon$, $\delta$, Bounds: $\|\cX\|, \|\cY\|$. 
		}
		\STATE{1. Calculate the minimum eigenvalue $\lambda_{\text{min}}(X^TX)$.}
		\STATE{2.} Privately release $\tilde{\lambda}_{\text{min}} =  \max\left\{\lambda_{\text{min}} + \frac{\sqrt{\log(6/\delta)}}{\epsilon/3}\|\cX\|^2 Z  - \frac{\log(6/\delta)}{\epsilon/3}\|\cX\|^2, 0\right\}$, where $Z\sim \cN(0,1)$.
		\STATE{3. Set $\lambda = \max\{0, \frac{\sqrt{d \log(6/\delta)\log(2d^2/\rho) }\|\cX\|^2}{\epsilon/3}   - \tilde{\lambda}_{\min }\}$}
		\STATE{4. Privately release $\widehat{X^TX} = X^TX  + \frac{\sqrt{\log(6/\delta)}\|\cX\|^2}{\epsilon/3}Z$ for $Z\in\R^{d\times d}$  is a symmetric matrix and every element from the upper triangular matrix is sampled from $\cN(0,1)$. }
		\STATE{5. Privately release $\widehat{X\vct y} = X\vct y  + \frac{\sqrt{\log(6/\delta)}\|\cX\|\|\cY\|}{\epsilon/3}Z$ for $Z\sim \cN(0, I_d)$. }
		\OUTPUT{ $\tilde{\theta}=   (\widehat{X^TX}+ \lambda I)^{-1} \widehat{X\vct y}$}
	\end{algorithmic}
\end{algorithm}

In both \AdaSP{} and \AdaOPS{}, we choose $\lambda$ by minimizing an upper bound of
$
F(\tilde{\theta})   -  F(\theta^*) 
$
in the form of ``variance'' and ``bias''
$$
\tilde{O}(\frac{d\|\cX\|^4\|\theta^*\|^2}{\lambda + \lambda_{\min }})  + \lambda \|\theta^*\|^2.
$$
Note that while $\|\theta^*\|^2$ cannot be privately released in general due to unbounded sensitivity, it appears in both terms and do not enter the decision process of finding the optimal $\lambda$ that minimizes the bound. This convenient feature follows from our assumption that $\|\cY\|\lesssim \|\cX\|\|\theta^*\|$. Dealing with the general case involving an arbitrary $\|\cY\|$ is an intriguing open problem.

A tricky situation for \AdaOPS{} is that the choice of $\gamma$ depends on $\lambda$ through $\tilde{L}$, which is the local Lipschitz constant at the ridge regression solution $\theta^*_\lambda$. But the choice of $\lambda$ also depends on $\gamma$ since the ``variance'' term above is inversely proportional to $\gamma$.  Our solution is to express $\tilde{L}$ (hence $\gamma$) as a function of $\lambda$ and solve the nonlinear univariate optimization problem \eqref{eq:adachoice_of_lambda}.

We are now ready to state the main results. 
\begin{theorem}\label{thm:adaops}
	
	Algorithm~\ref{alg:self-tuning-AdaOPS} outputs $\tilde{\theta}$ which obeys that
	\begin{enumerate}
		\item[(i)] It satisfies $(\epsilon,\delta)$-DP.
		\item[(ii)] 	Assume $\|\cY\|\lesssim \|\cX\|\|\theta^*\|$. With probability $1-\varrho$, 
		$$F(\tilde{\theta})  - F(\theta^*)  \leq 
		O\left(\frac{\sqrt{d+\log(\frac{1}{\varrho})}\|\cX\|^2\|\theta^*\|^2}{\epsilon/\sqrt{\log(\frac{1}{\delta})}}\wedge \frac{d[d+\log(\frac{1}{\varrho})] \|\theta^*\|^2}{\alpha n\epsilon^2 / \log(\frac{1}{\delta})}\right).
		$$
		\item[(iii)] Assume that $\mathbf{y}|X$ obeys a linear Gaussian model and $X$ is full-rank. Then there is an event $E$ satisfying $\P(E)\geq 1-\delta/3$ and $E\independent \mathbf{y} | X$, such that 
		\begin{align*}
			\E[\tilde{\theta} | X,E]  = \theta_0 &&\text{ and }&&\Cov[\tilde{\theta} |X,E]  \prec  \left(1 + O\left(\frac{ \tilde{C} d \log(6/\delta)}{\sigma^2 \alpha n \epsilon^2}\right) \right)\sigma^2(X^TX)^{-1}
		\end{align*}
		where constant 
		$$\tilde{C} :=\|\cY\|^2+\|\cX\|^2(\|\theta_0\|^2
		+\sigma^2\tr[(X^TX)^{-1}]).$$
	\end{enumerate}
\end{theorem}
The proof, deferred to Appendix~\ref{sec:proof_OPS}, makes use of a fine-grained DP-analysis through the recent per instance DP techniques \citep{wang2017per} and then convert the results to DP by releasing data dependent bounds of $\alpha$ and the magnitude of a ridge-regression output $\theta^*_\lambda$ with an adaptively chosen $\lambda$. Note that $\|\theta^*_\lambda\|$ does not have a bounded global sensitivity. The method to release it differentially privately (described in Lemma~\ref{lem:lipschitz_GS}) is part of our technical contribution.

The \AdaSP{} algorithm is simpler and enjoys slightly stronger theoretical guarantees.
\begin{theorem}\label{thm:adaSuffP}
	Algorithm~\ref{alg:adaSuffP} outputs $\tilde{\theta}$ which obeys that
	\begin{enumerate}
		\item[(i)] It satisfies $(\epsilon,\delta)$-DP.
		\item[(ii)] Assume $\|\cY\|\lesssim \|\cX\|\|\theta^*\|$.  	With probability $1-\varrho$, 
		$$
		F(\tilde{\theta})  - F(\theta^*) \leq  
		O\left(\frac{  \sqrt{d\log(\frac{d^2}{\varrho})}\|\cX\|^2\|\theta^*\|^2  }{\epsilon/\sqrt{\log(\frac{6}{\delta})}} \wedge \frac{  \|\cX\|^4\|\theta^*\|^2 \tr[(X^TX)^{-1}]}{\epsilon^2/ [\log(\frac{6}{\delta})\log(\frac{d^2}{\varrho}) ]}\right)
		$$
		\item[(iii)] 
		Assume that $\mathbf{y}|X$ obeys a linear Gaussian model and $X$ has a sufficiently large $\alpha$. Then  there is an event $E$ satisfying $\P(E)\geq 1-\delta/3$ and $E\independent \mathbf{y} | X$, such that 
		$
		\E[\tilde{\theta} | X, E]  = \theta_0
		$
		and
		\begin{align*}
		\E[\|\tilde{\theta}-\theta_0\|^2 | X, E]
		=\sigma^2\tr[(X^TX)^{-1}] + O\left(\frac{\tilde{C} \|\cX\|^2\tr[(X^TX)^{-2}] }{\epsilon^2/\log(\frac{6}{\delta})}  \right),
		\end{align*}
		with the same constant $\tilde{C}$ in Theorem~\ref{thm:adaops} (iii).
	\end{enumerate}
\end{theorem}
The proof of  Statement (1) is straightforward. Note that we release the eigenvalue $\lambda_{\text{min}}(X^TX)$, $X\vct y$ and $X^TX$  differentially privately each with parameter $(\epsilon/3,\delta/3)$. For the first two, we use Gaussian mechanism and for $X^TX$, we use the Analyze-Gauss algorithm \citep{dwork2014analyze} with a symmetric Gaussian random matrix. The result then follows from the composition theorem of differential privacy. The proof of the second and third statements is provided in Appendix~\ref{app:suffpert}. The main technical challenge is to prove the concentration on the spectrum and the Johnson-Lindenstrauss-like distance preserving properties for symmetric Gaussian random matrices (Lemma~\ref{lem:JL-ellipsoid}). We note that while \SuffP{} is an old algorithm the analysis of its theoretical properties is new to this paper.
\paragraph{Remarks.} Both \AdaOPS{} and \AdaSP{} match the smaller of the two lower bounds \eqref{eq:dp_lowerbound_lipschitz} and \eqref{eq:dp_lowerbound_strongcvx} for each problem instance. They are slightly different in that \AdaOPS{} preserves the shape of the intrinsic geometry while \AdaSP{}'s bounds are slightly stronger as they do not explicitly depend on the smallest eigenvalue. 


\section{Experiments}\label{sec:exp}
In this section, we conduct synthetic and real data experiments to benchmark the performance of \AdaOPS{} and \AdaSP{} relative to existing algorithms we discussed in Section~\ref{sec:priorwork}. \NSGD{}  and Sub-Agg are excluded because they are dominated by \OBP{} and an $(\epsilon,\delta)$-DP version of \OPS{}, which we describe in Appendix~\ref{app:epsdelta_OPS}. 
The code to reproduce all experimental results in this paper is available at \url{https://github.com/yuxiangw/optimal_dp_linear_regression}.


\begin{figure}[t]
	\centering
	\begin{subfigure}[t]{0.45\textwidth}
		\centering
		\includegraphics[width=\textwidth]{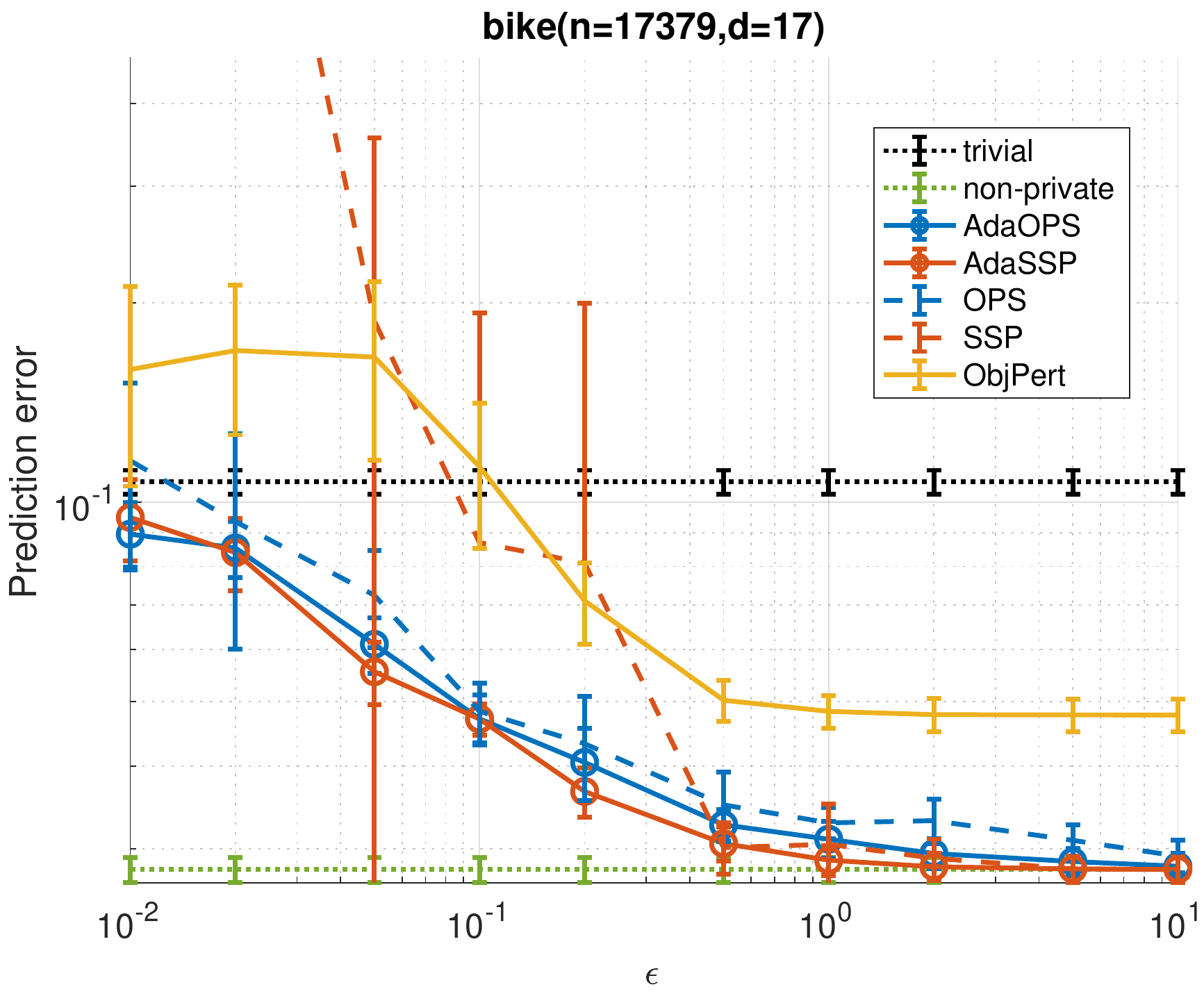}
	\end{subfigure}
	\begin{subfigure}[t]{0.45\textwidth}
		\centering
		\includegraphics[width=\textwidth]{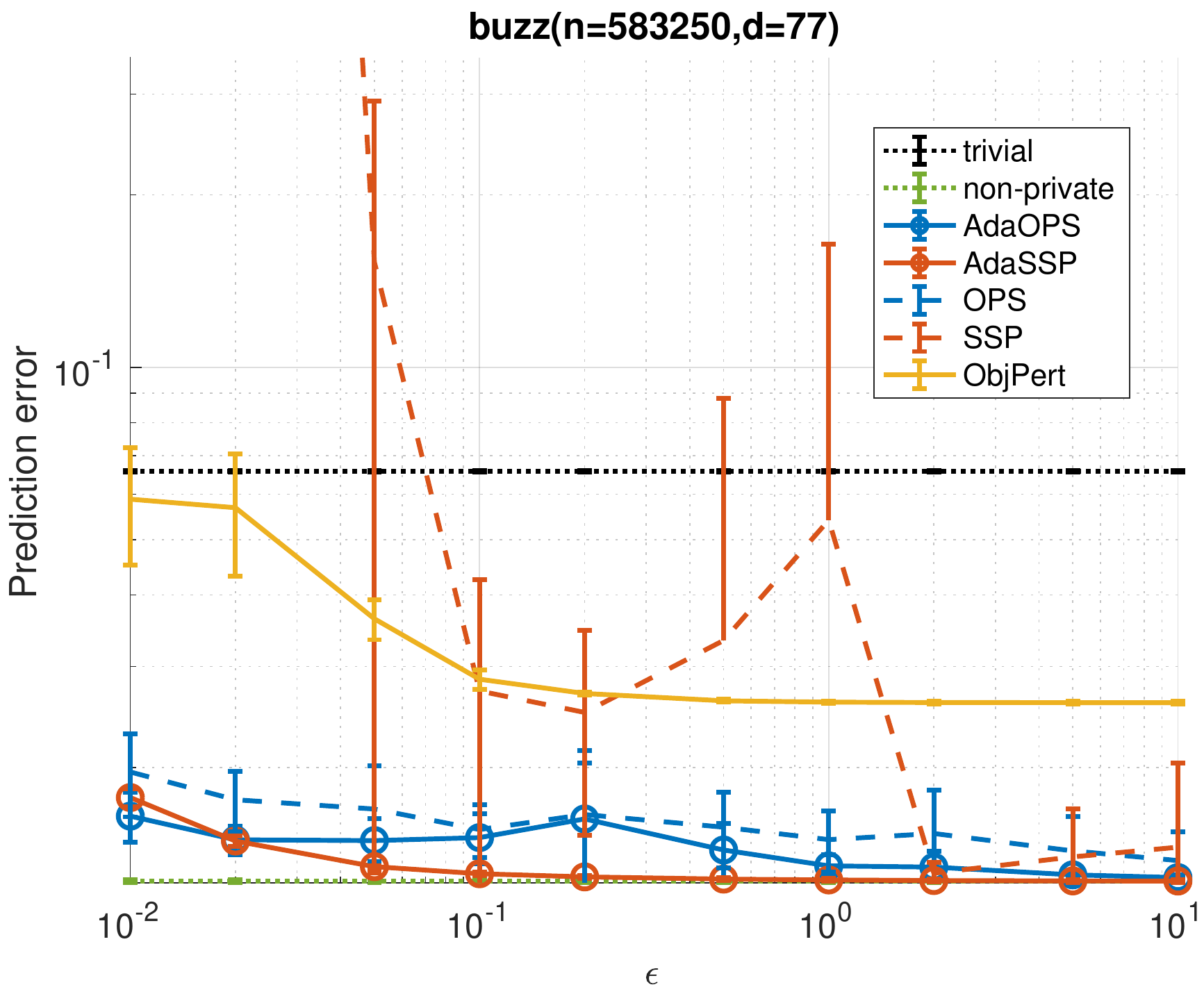}
	\end{subfigure}
	\begin{subfigure}[t]{0.45\textwidth}
		\centering
		\includegraphics[width=\textwidth]{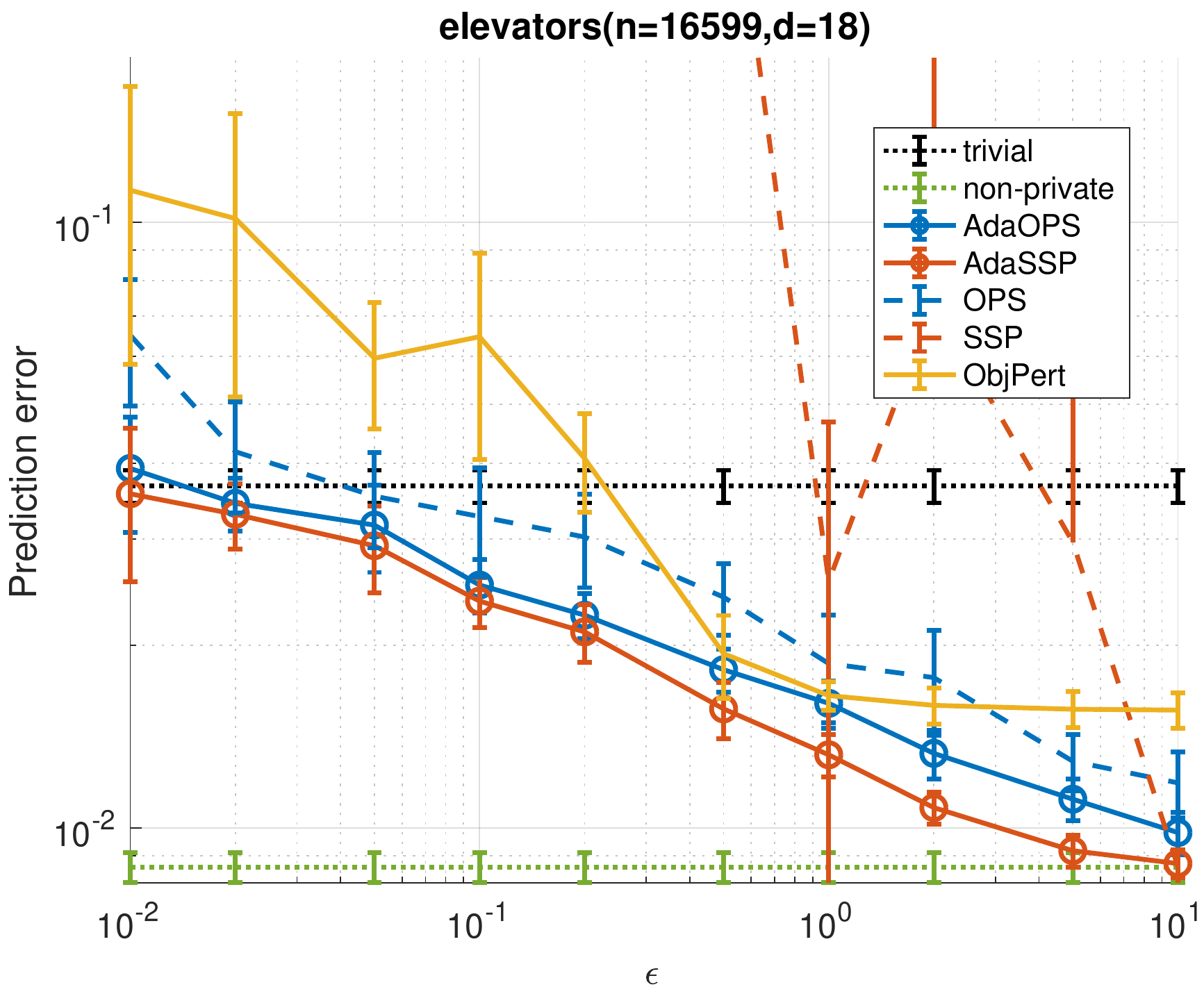}
	\end{subfigure}
	\begin{subfigure}[t]{0.45\textwidth}
		\centering
		\includegraphics[width=\textwidth]{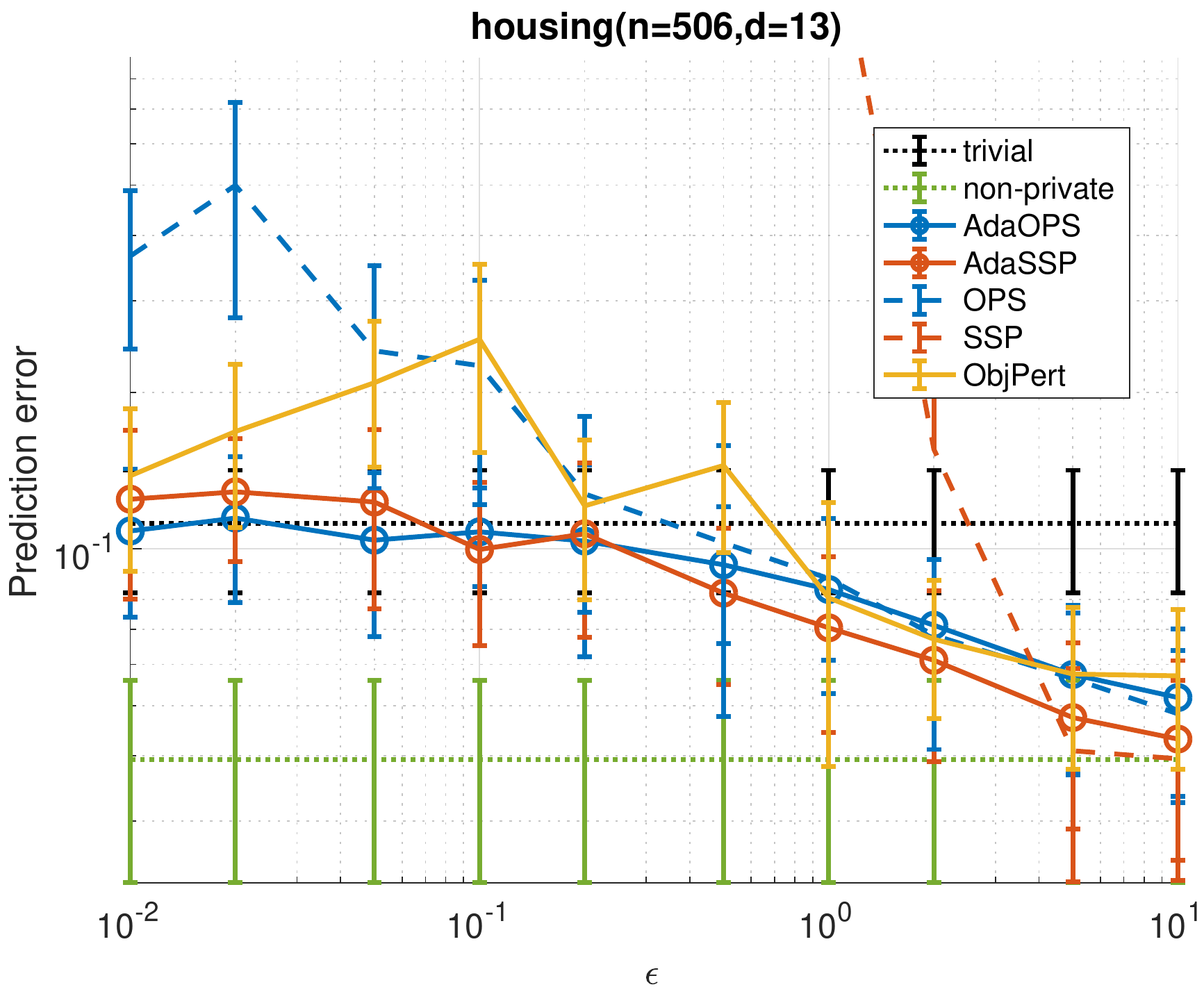}
	\end{subfigure}
	\caption{\small Example of results of differentially private linear regression algorithms on UCI data sets for a sequence of $\epsilon$. Reported on the y-axis is the cross-validation prediction error in MSE and their confidence intervals. }\label{fig:UCI_dataset}
	\end{figure}
	\begin{figure}	[t]
	\centering
	\begin{subfigure}[t]{0.45\textwidth}
		\centering
		\includegraphics[width=\textwidth]{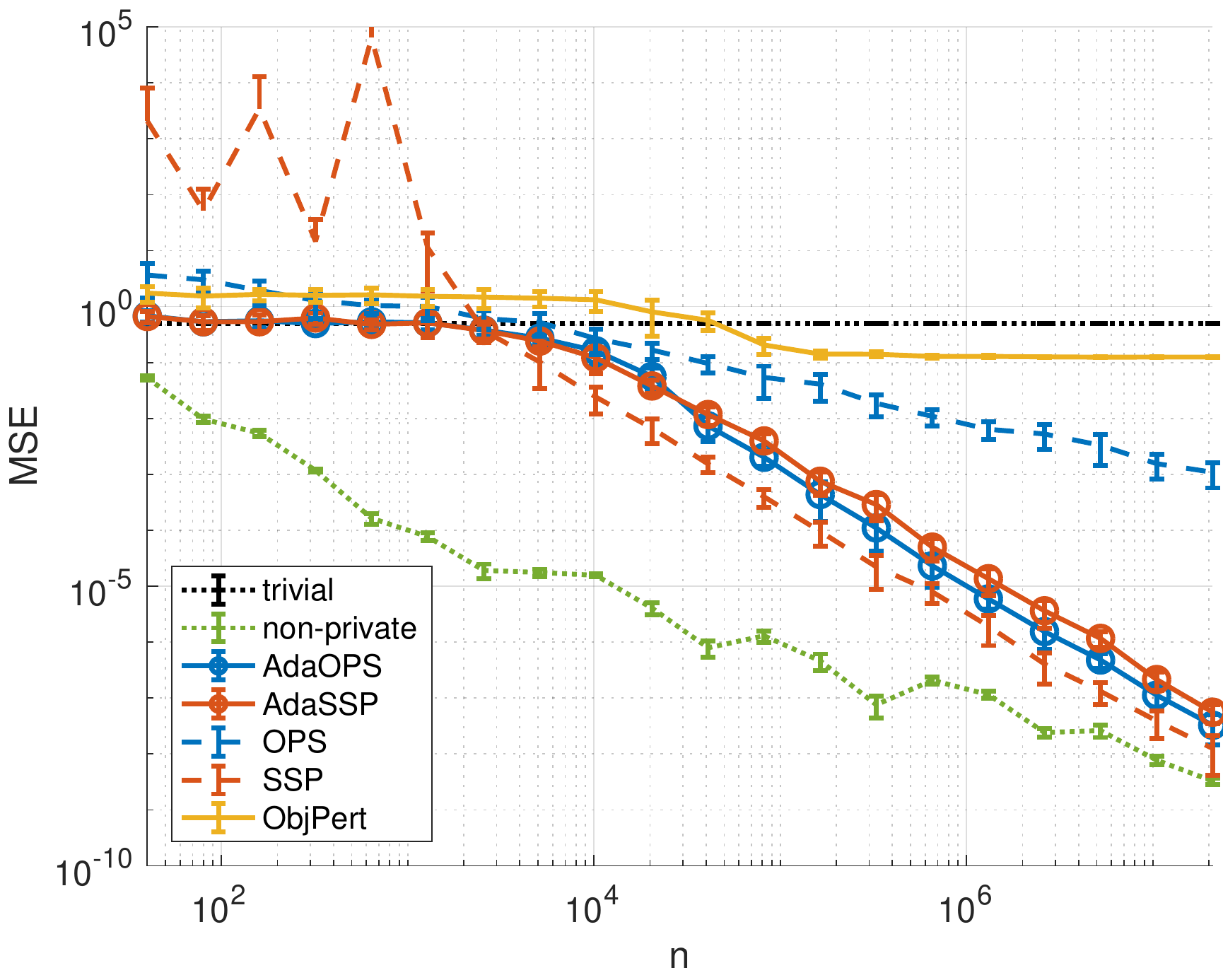}
		\caption{Estimation MSE at $\epsilon=0.1$}	
	\end{subfigure}
	\begin{subfigure}[t]{0.45\textwidth}
		\centering
		\includegraphics[width=\textwidth]{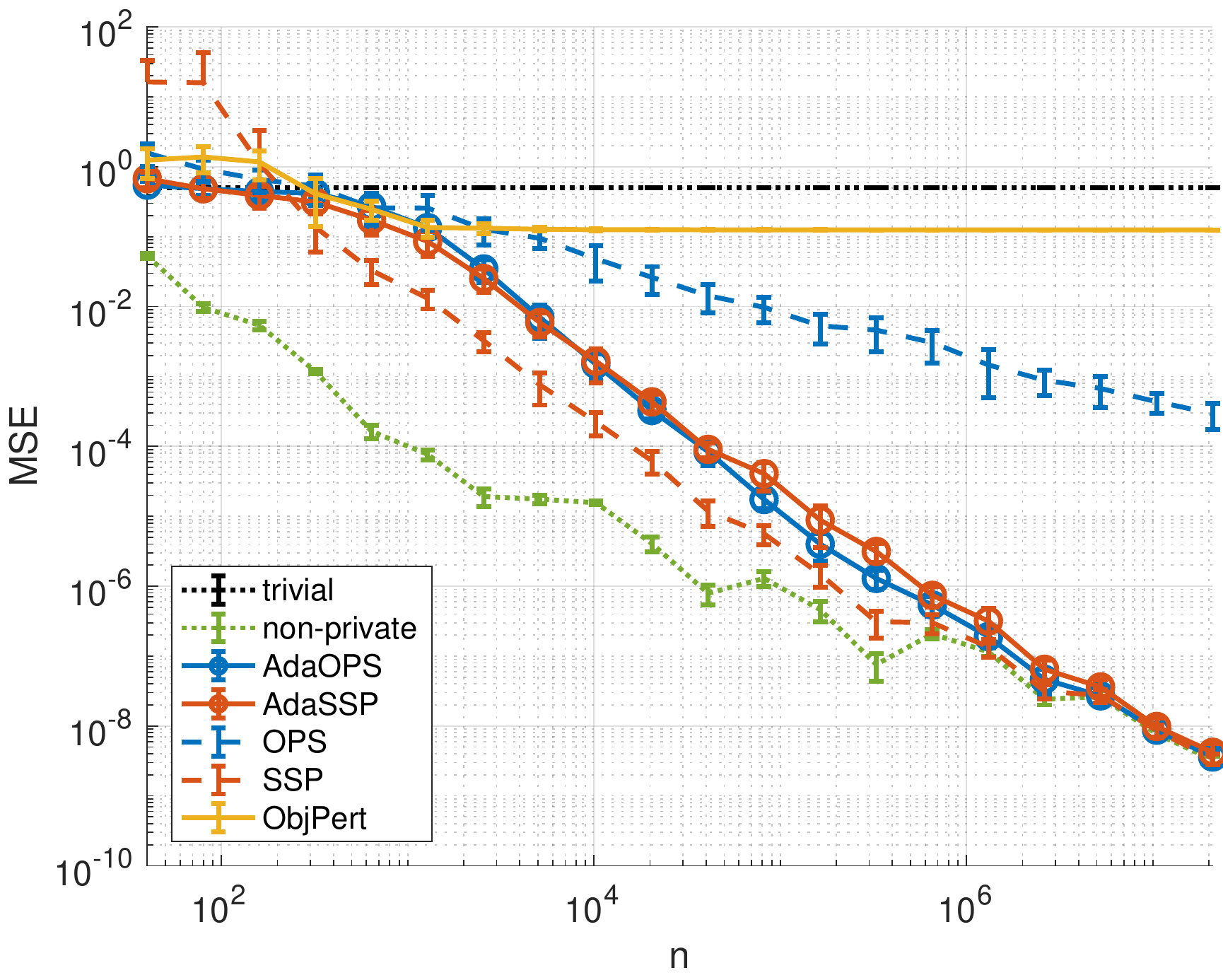}
		\caption{Estimation MSE at $\epsilon=1$}	
	\end{subfigure}
	\begin{subfigure}[t]{0.45\textwidth}
		\centering
		\includegraphics[width=\textwidth]{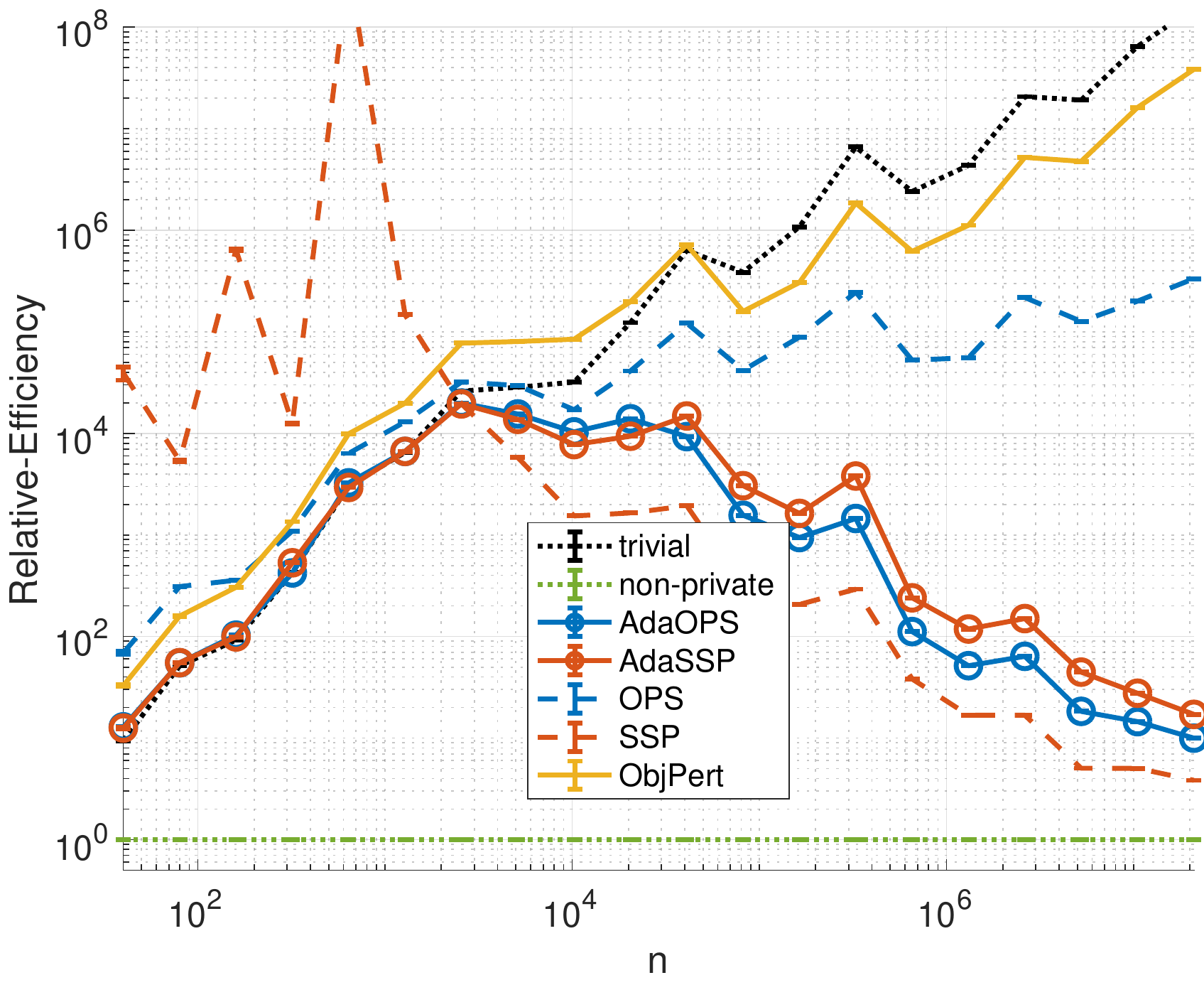}
		\caption{Rel. efficiency at $\epsilon=0.1$}
	\end{subfigure}
	\begin{subfigure}[t]{0.45\textwidth}
		\centering
		\includegraphics[width=\textwidth]{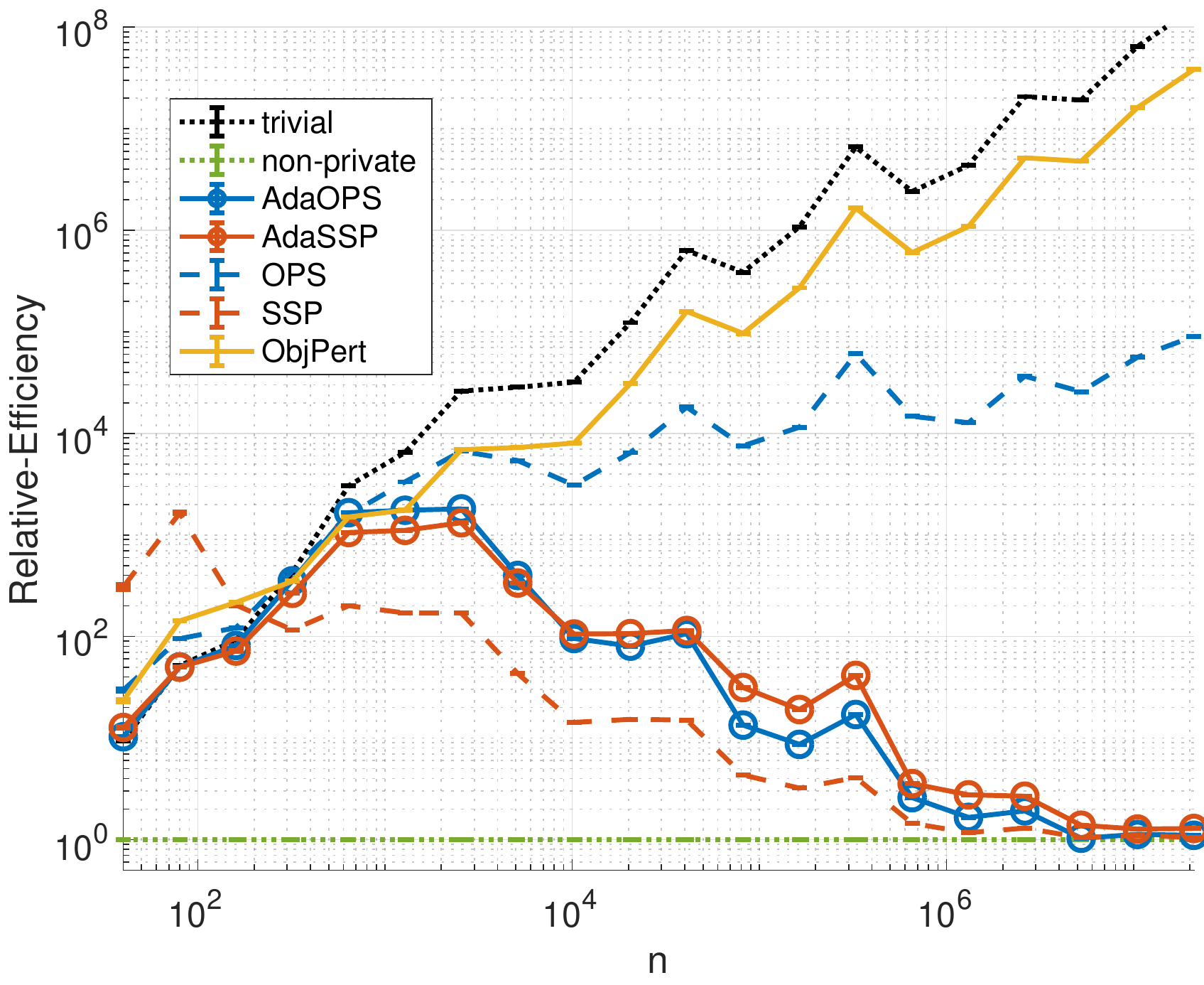}
		\caption{Rel. efficiency at $\epsilon=1$}	
	\end{subfigure}
	\caption{\small Example of differentially private linear regression under linear Gaussian model with an increasing data size $n$. We simulate the data from $d=10$, $\theta_0$ drawn from a uniform distribution defined on $[0,1]^d$. We generate $X\in \R^{n\times d}$ as a Gaussian random matrix and then generate $y \sim \cN(X\theta_0, I_d)$.  We used $\epsilon = 1$ and $\epsilon=0.1$, both with $\delta = 1/n^2$. The results clearly illustrate the asymptotic efficiency of the proposed approaches.}\label{fig:gaussian}
\end{figure}

\paragraph{Prediction accuracy in UCI data sets experiments. }
The first set of experiments is on training linear regression on a number of UCI regression data sets.
Standard $z$-scoring are performed and all data points are normalized to having an Euclidean norm of $1$ as a preprocessing step.
Results on four of the data sets are presented in Figure~\ref{fig:UCI_dataset}. As we can see, \SuffP{} is unstable for small data.  \OBP{} suffers from a pre-defined bound $\|\Theta\|$ and does not converge to nonprivate solution even with a large $\epsilon$. \OPS{} performs well but still does not take advantage of the strong convexity that is intrinsic to the data set.  \AdaOPS{} and \AdaSP{} on the other hand are able to nicely interpolate between the trivial solution and the non-private baseline and performed as well as or better than baselines for all $\epsilon$.  More detailed quantitative results on all the 36 UCI data sets are presented in Table~\ref{tab:uci_eps=0.1} and Table~\ref{tab:uci_eps=1} for $\epsilon=0.1,\delta=\min\{1e-6,1/n^2\}$ and $\epsilon = 1,\delta=\min\{1e-6,1/n^2\}$ respectively  in Appendix~\ref{sec:realdata_table}.


%



\paragraph{Parameter estimation under linear Gaussian model. }
To illustrate the performance of the algorithms under standard statistical assumptions, we also benchmarked the algorithms on synthetic data generated by a linear Gaussian model. The results, shown in Figure~\ref{fig:gaussian} illustrates that as $n$ gets large, \AdaOPS{} and \AdaSP{} with $\epsilon=0.1$ and $\epsilon = 1$ converge to the  maximum likelihood estimator at a rate faster than the optimal statistical rate that MLE estimates $\theta^*$, therefore at least for large $n$, differential privacy comes for free. Note that there is a gap in \SuffP{} and \AdaSP{} for large $n$, this can be thought of as a cost of adaptivity as \AdaSP{} needs to spend some portion of its privacy budget to release $\lambda_{\min }$, which \SuffP{} does not, this can be fixed by using more careful splitting of the privacy budget.

\section{ Conclusion}
In this paper, we presented a detailed case-study of the problem of differentially private linear regression. We clarified the relationships between various quantities of the problems as they appear in the private and non-private information-theoretic lower bounds. We also surveyed the existing algorithms and highlighted that the main drawback using these algorithms relative to their non-private counterpart is that they cannot adapt to data-dependent quantities. This is particularly true for linear regression where the ordinary least square algorithm is able to work optimally for a large class of different settings.

We proposed \AdaOPS{} and \AdaSP{} to address the issue and showed that they both work in unbounded domain. Moreover, they smoothly interpolate the two regimes studied in \citet{bassily2014private} and behave nearly optimally for every instance. We tested the two algorithms on 36 real-life data sets from the UCI machine learning repository and we see significant improvement over popular algorithms for almost all configurations of $\epsilon$. 

Future work includes extending the result beyond linear regression and releasing off-the-shelf packages for adaptive differentially private learning.

\section*{Acknowledgements}
The author thanks the anonymous reviewers for helpful feedbacks and Zichao Yang for sharing the 36 UCI regression data sets as was used in \citep{yang2015carte}. 


\bibliography{FreeDP}
\bibliographystyle{apa-good}

\newpage
\appendix
\section{Results on the 36 real  regression data sets in UCI repository}\label{sec:realdata_table}
The detailed results on the 36 UCI data sets are presented in Table~\ref{tab:uci_eps=0.1} for $\epsilon=0.1,\delta=\min\{1e-6,1/n^2\}$ and Table~\ref{tab:uci_eps=1} for $\epsilon=1, \delta=\min\{1e-6,1/n^2\}$. The boldface denotes the DP algorithm where the standard deviation is smaller than the error (a positive quantity), and the 95\%  confidence interval covers the observed best performance among benchmarked DP algorithms.

\begin{table*}[h!]
	\caption{Summary of UCI data experiments at $\epsilon = 0.1$. The boldface denotes the DP algorithm where the standard deviation is smaller than the error (a positive quantity), and the 95\%  confidence interval covers the observed best performance among benchmarked DP algorithms.}\label{tab:uci_eps=0.1}
	\resizebox{\textwidth}{!}{
		\begin{tabular}{|l|l|l|l|l|l|l|l|}
			\hline 
			& Trivial & non-private  & \OBP{}&\OPS{} &\SuffP{}  & \AdaOPS{} & \AdaSP{} \\ 
			\hline
			3droad&0.0275$\pm$0.00014&0.0265$\pm$0.00012&0.0267$\pm$0.00013&0.027$\pm$0.00026&\textbf{0.0265}$\pm$\textbf{0.00019}&\textbf{0.0265}$\pm$\textbf{0.00019}&\textbf{0.0265}$\pm$\textbf{0.00019}\\  
			airfoil&0.103$\pm$0.0069&0.0533$\pm$0.0074&0.356$\pm$0.064&\textbf{0.138}$\pm$\textbf{0.086}&0.232$\pm$0.28&\textbf{0.0914}$\pm$\textbf{0.015}&\textbf{0.0878}$\pm$\textbf{0.014}\\  
			autompg&0.113$\pm$0.011&0.0221$\pm$0.0032&\textbf{0.143}$\pm$\textbf{0.096}&0.242$\pm$0.11&5.44$\pm$6.1&\textbf{0.098}$\pm$\textbf{0.03}&\textbf{0.115}$\pm$\textbf{0.047}\\  
			autos&0.13$\pm$0.042&0.0274$\pm$0.011&\textbf{0.17}$\pm$\textbf{0.13}&0.308$\pm$0.13&1.7e+03$\pm$2.5e+03&\textbf{0.136}$\pm$\textbf{0.066}&\textbf{0.132}$\pm$\textbf{0.064}\\  
			bike&0.107$\pm$0.0028&0.0279$\pm$0.00078&0.113$\pm$0.018&\textbf{0.0484}$\pm$\textbf{0.005}&0.0869$\pm$0.067&\textbf{0.0471}$\pm$\textbf{0.004}&\textbf{0.0471}$\pm$\textbf{0.0026}\\  
			breastcancer&0.194$\pm$0.027&0.139$\pm$0.025&\textbf{0.212}$\pm$\textbf{0.078}&\textbf{0.269}$\pm$\textbf{0.13}&9.54e+03$\pm$1.9e+04&\textbf{0.204}$\pm$\textbf{0.037}&\textbf{0.196}$\pm$\textbf{0.051}\\  
			buzz&0.0658$\pm$0.00015&0.0127$\pm$4.6e-05&0.0285$\pm$0.00071&0.0156$\pm$0.001&0.0272$\pm$0.0097&0.0151$\pm$0.00095&\textbf{0.013}$\pm$\textbf{9.7e-05}\\  
			challenger&0.141$\pm$0.084&0.138$\pm$0.088&0.323$\pm$0.28&0.338$\pm$0.13&3.07$\pm$3.9&\textbf{0.159}$\pm$\textbf{0.13}&\textbf{0.146}$\pm$\textbf{0.093}\\  
			concrete&0.127$\pm$0.0043&0.0445$\pm$0.0033&0.237$\pm$0.076&0.181$\pm$0.042&1.94$\pm$1.8&\textbf{0.12}$\pm$\textbf{0.011}&\textbf{0.119}$\pm$\textbf{0.016}\\  
			concreteslump&0.149$\pm$0.039&0.0245$\pm$0.0071&0.349$\pm$0.094&0.549$\pm$0.24&3.14$\pm$2.5&\textbf{0.151}$\pm$\textbf{0.064}&\textbf{0.165}$\pm$\textbf{0.065}\\  
			elevators&0.0367$\pm$0.0014&0.00861$\pm$0.00031&0.0647$\pm$0.015&0.0327$\pm$0.0042&0.645$\pm$0.98&\textbf{0.0252}$\pm$\textbf{0.0026}&\textbf{0.0237}$\pm$\textbf{0.0022}\\  
			energy&0.235$\pm$0.012&0.0232$\pm$0.0023&0.332$\pm$0.09&\textbf{0.161}$\pm$\textbf{0.083}&1.7e+03$\pm$3.4e+03&\textbf{0.167}$\pm$\textbf{0.034}&\textbf{0.15}$\pm$\textbf{0.032}\\  
			fertility&0.0977$\pm$0.024&0.0863$\pm$0.024&0.203$\pm$0.04&0.639$\pm$0.16&439$\pm$8.6e+02&\textbf{0.108}$\pm$\textbf{0.048}&\textbf{0.115}$\pm$\textbf{0.032}\\  
			forest&0.0564$\pm$0.0081&0.0571$\pm$0.0086&0.12$\pm$0.022&0.177$\pm$0.036&41.9$\pm$77&\textbf{0.0622}$\pm$\textbf{0.017}&\textbf{0.0675}$\pm$\textbf{0.013}\\  
			gas&0.112$\pm$0.0062&0.0214$\pm$0.0028&0.109$\pm$0.015&\textbf{0.0546}$\pm$\textbf{0.012}&0.923$\pm$0.63&0.0801$\pm$0.0078&0.0875$\pm$0.0073\\  
			houseelectric&0.122$\pm$0.00017&0.0136$\pm$1.4e-05&0.0409$\pm$0.00027&0.0144$\pm$0.00017&\textbf{0.0136}$\pm$\textbf{2.2e-05}&\textbf{0.0136}$\pm$\textbf{2.2e-05}&\textbf{0.0136}$\pm$\textbf{2.2e-05}\\  
			housing&0.112$\pm$0.019&0.0394$\pm$0.01&0.253$\pm$0.063&0.225$\pm$0.065&2.24$\pm$2.3&\textbf{0.108}$\pm$\textbf{0.023}&\textbf{0.0997}$\pm$\textbf{0.035}\\  
			keggdirected&0.117$\pm$0.00095&0.0188$\pm$0.0011&0.0637$\pm$0.0042&0.0266$\pm$0.0019&0.23$\pm$0.33&0.0227$\pm$0.0015&\textbf{0.0212}$\pm$\textbf{0.0011}\\  
			keggundirected&0.0694$\pm$0.00074&0.00475$\pm$8.9e-05&0.0365$\pm$0.0028&0.0166$\pm$0.0033&0.353$\pm$0.4&0.0107$\pm$0.0012&\textbf{0.00912}$\pm$\textbf{0.00046}\\  
			kin40k&0.0634$\pm$0.0012&0.0632$\pm$0.0013&0.0871$\pm$0.0092&0.0717$\pm$0.0026&\textbf{0.0633}$\pm$\textbf{0.002}&\textbf{0.0639}$\pm$\textbf{0.0021}&\textbf{0.064}$\pm$\textbf{0.0021}\\  
			machine&0.121$\pm$0.013&0.0395$\pm$0.0051&0.282$\pm$0.14&0.347$\pm$0.14&2.27e+03$\pm$4.5e+03&\textbf{0.105}$\pm$\textbf{0.025}&\textbf{0.141}$\pm$\textbf{0.068}\\  
			parkinsons&0.17$\pm$0.0026&0.128$\pm$0.0024&0.211$\pm$0.014&\textbf{0.157}$\pm$\textbf{0.011}&132$\pm$2.6e+02&\textbf{0.159}$\pm$\textbf{0.0065}&\textbf{0.156}$\pm$\textbf{0.0064}\\  
			pendulum&0.0226$\pm$0.0061&0.0181$\pm$0.0049&0.118$\pm$0.027&0.122$\pm$0.041&24.8$\pm$45&\textbf{0.0276}$\pm$\textbf{0.011}&0.0346$\pm$0.0069\\  
			pol&0.345$\pm$0.0028&0.135$\pm$0.0023&0.302$\pm$0.032&\textbf{0.196}$\pm$\textbf{0.02}&281$\pm$5.3e+02&0.214$\pm$0.0056&0.214$\pm$0.0061\\  
			protein&0.167$\pm$0.0011&0.119$\pm$0.0014&0.158$\pm$0.01&0.137$\pm$0.0044&\textbf{0.149}$\pm$\textbf{0.06}&0.129$\pm$0.0015&\textbf{0.125}$\pm$\textbf{0.0026}\\  
			pumadyn32nm&0.0935$\pm$0.0039&0.0941$\pm$0.0039&0.124$\pm$0.0046&0.111$\pm$0.005&8.92e+03$\pm$1.8e+04&\textbf{0.0968}$\pm$\textbf{0.0065}&\textbf{0.0966}$\pm$\textbf{0.0063}\\  
			servo&0.184$\pm$0.039&0.0752$\pm$0.022&0.366$\pm$0.077&0.574$\pm$0.26&2.03$\pm$1.5&\textbf{0.195}$\pm$\textbf{0.065}&\textbf{0.198}$\pm$\textbf{0.081}\\  
			skillcraft&0.0439$\pm$0.0021&0.0203$\pm$0.0017&0.0817$\pm$0.013&0.0519$\pm$0.0099&4.72$\pm$4.3&\textbf{0.037}$\pm$\textbf{0.008}&\textbf{0.039}$\pm$\textbf{0.0056}\\  
			slice&0.196$\pm$0.0021&0.0283$\pm$0.00051&0.174$\pm$0.0053&\textbf{0.0924}$\pm$\textbf{0.0035}&11.2$\pm$9.4&0.0992$\pm$0.0021&0.132$\pm$0.0015\\  
			sml&0.211$\pm$0.0089&0.0143$\pm$0.00066&0.23$\pm$0.03&\textbf{0.0955}$\pm$\textbf{0.029}&59.9$\pm$80&0.134$\pm$0.0075&0.147$\pm$0.013\\  
			solar&0.0118$\pm$0.0042&0.0106$\pm$0.0038&0.0994$\pm$0.023&0.0667$\pm$0.017&5.95$\pm$9.6&\textbf{0.0165}$\pm$\textbf{0.0062}&\textbf{0.0204}$\pm$\textbf{0.0073}\\  
			song&0.0917$\pm$0.0003&0.0636$\pm$0.00033&0.0838$\pm$0.0014&0.072$\pm$0.00035&\textbf{0.0644}$\pm$\textbf{0.0005}&0.0685$\pm$0.00045&0.0697$\pm$0.00029\\  
			stock&0.0583$\pm$0.0095&0.013$\pm$0.0023&0.122$\pm$0.026&0.157$\pm$0.055&46.8$\pm$66&\textbf{0.0582}$\pm$\textbf{0.023}&\textbf{0.0651}$\pm$\textbf{0.024}\\  
			tamielectric&0.334$\pm$0.002&0.334$\pm$0.0021&0.341$\pm$0.0021&0.343$\pm$0.0065&\textbf{0.335}$\pm$\textbf{0.0033}&\textbf{0.337}$\pm$\textbf{0.0047}&\textbf{0.335}$\pm$\textbf{0.0033}\\  
			wine&0.0566$\pm$0.0028&0.0202$\pm$0.00099&0.153$\pm$0.028&0.0911$\pm$0.016&11.7$\pm$17&\textbf{0.058}$\pm$\textbf{0.011}&\textbf{0.0599}$\pm$\textbf{0.01}\\  
			yacht&0.105$\pm$0.017&0.0176$\pm$0.0055&0.273$\pm$0.076&0.371$\pm$0.14&4.92$\pm$6.8&\textbf{0.0967}$\pm$\textbf{0.035}&\textbf{0.109}$\pm$\textbf{0.03}\\  
			\hline
		\end{tabular} 
	}
\end{table*}

\begin{table}[h!]
	\caption{Summary of UCI data experiments at $\epsilon = 1$}\label{tab:uci_eps=1}
	\resizebox{\textwidth}{!}{
		\begin{tabular}{|l|l|l|l|l|l|l|l|}
			\hline 
			& Trivial & non-private  & \OBP{}&\OPS{} &\SuffP{}  & \AdaOPS{} & \AdaSP{} \\ 
			\hline
			3droad&0.0275$\pm$0.00014&0.0265$\pm$0.00012&0.0267$\pm$0.00013&0.0266$\pm$0.00012&\textbf{0.0265}$\pm$\textbf{0.00019}&\textbf{0.0265}$\pm$\textbf{0.00019}&\textbf{0.0265}$\pm$\textbf{0.00019}\\  
			airfoil&0.103$\pm$0.0069&0.0533$\pm$0.0074&0.0681$\pm$0.0074&0.0674$\pm$0.011&\textbf{0.0535}$\pm$\textbf{0.013}&0.0686$\pm$0.006&\textbf{0.0585}$\pm$\textbf{0.012}\\  
			autompg&0.113$\pm$0.011&0.0221$\pm$0.0032&0.0651$\pm$0.0072&0.0783$\pm$0.027&0.169$\pm$0.1&0.0522$\pm$0.0074&\textbf{0.044}$\pm$\textbf{0.016}\\  
			autos&0.13$\pm$0.042&0.0274$\pm$0.011&\textbf{0.0868}$\pm$\textbf{0.044}&\textbf{0.0761}$\pm$\textbf{0.038}&2.07$\pm$1.6&\textbf{0.108}$\pm$\textbf{0.076}&\textbf{0.0971}$\pm$\textbf{0.053}\\  
			bike&0.107$\pm$0.0028&0.0279$\pm$0.00078&0.0484$\pm$0.0017&0.0328$\pm$0.0011&\textbf{0.0305}$\pm$\textbf{0.0045}&0.031$\pm$0.0012&\textbf{0.0288}$\pm$\textbf{0.0015}\\  
			breastcancer&0.194$\pm$0.027&0.139$\pm$0.025&\textbf{0.181}$\pm$\textbf{0.047}&\textbf{0.198}$\pm$\textbf{0.083}&25.1$\pm$37&\textbf{0.186}$\pm$\textbf{0.036}&\textbf{0.184}$\pm$\textbf{0.038}\\  
			buzz&0.0658$\pm$0.00015&0.0127$\pm$4.6e-05&0.026$\pm$8e-05&0.015$\pm$0.0011&0.0541$\pm$0.069&0.0135$\pm$0.00038&\textbf{0.0127}$\pm$\textbf{7.1e-05}\\  
			challenger&0.141$\pm$0.084&0.138$\pm$0.088&0.181$\pm$0.13&0.529$\pm$0.35&5.27$\pm$8.1&\textbf{0.142}$\pm$\textbf{0.13}&\textbf{0.145}$\pm$\textbf{0.14}\\  
			concrete&0.127$\pm$0.0043&0.0445$\pm$0.0033&0.0759$\pm$0.0077&0.084$\pm$0.013&\textbf{0.0569}$\pm$\textbf{0.03}&0.0811$\pm$0.0044&0.0658$\pm$0.0051\\  
			concreteslump&0.149$\pm$0.039&0.0245$\pm$0.0071&\textbf{0.197}$\pm$\textbf{0.18}&\textbf{0.177}$\pm$\textbf{0.076}&0.27$\pm$0.11&\textbf{0.143}$\pm$\textbf{0.053}&\textbf{0.138}$\pm$\textbf{0.027}\\  
			elevators&0.0367$\pm$0.0014&0.00861$\pm$0.00031&0.0165$\pm$0.00057&0.0187$\pm$0.0024&\textbf{0.0255}$\pm$\textbf{0.021}&0.0161$\pm$0.00091&\textbf{0.0132}$\pm$\textbf{0.0011}\\  
			energy&0.235$\pm$0.012&0.0232$\pm$0.0023&0.086$\pm$0.0061&\textbf{0.0596}$\pm$\textbf{0.025}&0.0983$\pm$0.065&0.0675$\pm$0.0084&\textbf{0.051}$\pm$\textbf{0.0094}\\  
			fertility&0.0977$\pm$0.024&0.0863$\pm$0.024&0.182$\pm$0.058&0.185$\pm$0.055&2.81$\pm$3&\textbf{0.102}$\pm$\textbf{0.043}&\textbf{0.112}$\pm$\textbf{0.055}\\  
			forest&0.0564$\pm$0.0081&0.0571$\pm$0.0086&0.0774$\pm$0.0092&0.0802$\pm$0.0099&0.13$\pm$0.041&\textbf{0.0593}$\pm$\textbf{0.012}&\textbf{0.0585}$\pm$\textbf{0.0094}\\  
			gas&0.112$\pm$0.0062&0.0214$\pm$0.0028&0.0593$\pm$0.0044&\textbf{0.0432}$\pm$\textbf{0.0033}&5.64$\pm$9&\textbf{0.0471}$\pm$\textbf{0.0068}&\textbf{0.047}$\pm$\textbf{0.0063}\\  
			houseelectric&0.122$\pm$0.00017&0.0136$\pm$1.4e-05&0.0406$\pm$6.3e-05&0.0138$\pm$7e-05&\textbf{0.0136}$\pm$\textbf{2.2e-05}&\textbf{0.0136}$\pm$\textbf{2.2e-05}&\textbf{0.0136}$\pm$\textbf{2.2e-05}\\  
			housing&0.112$\pm$0.019&0.0394$\pm$0.01&\textbf{0.0805}$\pm$\textbf{0.042}&0.0877$\pm$0.017&1.89$\pm$2.4&\textbf{0.0835}$\pm$\textbf{0.031}&\textbf{0.0705}$\pm$\textbf{0.026}\\  
			keggdirected&0.117$\pm$0.00095&0.0188$\pm$0.0011&0.0435$\pm$0.00054&0.0234$\pm$0.0022&\textbf{0.0289}$\pm$\textbf{0.019}&0.0215$\pm$0.0018&\textbf{0.0192}$\pm$\textbf{0.00064}\\  
			keggundirected&0.0694$\pm$0.00074&0.00475$\pm$8.9e-05&0.0213$\pm$0.00023&0.00942$\pm$0.0016&0.0195$\pm$0.023&0.00633$\pm$0.00027&\textbf{0.00552}$\pm$\textbf{0.00014}\\  
			kin40k&0.0634$\pm$0.0012&0.0632$\pm$0.0013&\textbf{0.0633}$\pm$\textbf{0.002}&0.065$\pm$0.0012&\textbf{0.0632}$\pm$\textbf{0.002}&\textbf{0.0632}$\pm$\textbf{0.002}&\textbf{0.0633}$\pm$\textbf{0.002}\\  
			machine&0.121$\pm$0.013&0.0395$\pm$0.0051&0.104$\pm$0.016&\textbf{0.0825}$\pm$\textbf{0.027}&0.77$\pm$1.3&0.0809$\pm$0.013&\textbf{0.0671}$\pm$\textbf{0.016}\\  
			parkinsons&0.17$\pm$0.0026&0.128$\pm$0.0024&0.14$\pm$0.0019&0.142$\pm$0.004&4.21$\pm$8.1&\textbf{0.134}$\pm$\textbf{0.0026}&\textbf{0.133}$\pm$\textbf{0.0036}\\  
			pendulum&0.0226$\pm$0.0061&0.0181$\pm$0.0049&0.0426$\pm$0.01&0.0473$\pm$0.012&\textbf{0.0233}$\pm$\textbf{0.0066}&\textbf{0.0247}$\pm$\textbf{0.011}&\textbf{0.0233}$\pm$\textbf{0.0089}\\  
			pol&0.345$\pm$0.0028&0.135$\pm$0.0023&0.19$\pm$0.0026&0.145$\pm$0.0028&0.338$\pm$0.32&0.144$\pm$0.0031&\textbf{0.14}$\pm$\textbf{0.0033}\\  
			protein&0.167$\pm$0.0011&0.119$\pm$0.0014&0.131$\pm$0.0012&0.128$\pm$0.0047&\textbf{0.119}$\pm$\textbf{0.0022}&0.124$\pm$0.0036&\textbf{0.12}$\pm$\textbf{0.0021}\\  
			pumadyn32nm&0.0935$\pm$0.0039&0.0941$\pm$0.0039&\textbf{0.0948}$\pm$\textbf{0.0061}&0.101$\pm$0.0038&\textbf{0.0944}$\pm$\textbf{0.0065}&\textbf{0.0957}$\pm$\textbf{0.0065}&\textbf{0.0952}$\pm$\textbf{0.0066}\\  
			servo&0.184$\pm$0.039&0.0752$\pm$0.022&\textbf{0.152}$\pm$\textbf{0.077}&0.209$\pm$0.079&\textbf{0.126}$\pm$\textbf{0.072}&\textbf{0.149}$\pm$\textbf{0.051}&\textbf{0.124}$\pm$\textbf{0.06}\\  
			skillcraft&0.0439$\pm$0.0021&0.0203$\pm$0.0017&0.0298$\pm$0.0018&0.0325$\pm$0.0024&\textbf{0.0303}$\pm$\textbf{0.012}&\textbf{0.0268}$\pm$\textbf{0.0033}&\textbf{0.0247}$\pm$\textbf{0.0029}\\  
			slice&0.196$\pm$0.0021&0.0283$\pm$0.00051&0.0875$\pm$0.00082&0.0518$\pm$0.00099&100$\pm$1.8e+02&\textbf{0.0483}$\pm$\textbf{0.0013}&0.0556$\pm$0.00059\\  
			sml&0.211$\pm$0.0089&0.0143$\pm$0.00066&0.0751$\pm$0.0037&\textbf{0.0391}$\pm$\textbf{0.0042}&52.6$\pm$1e+02&0.0502$\pm$0.0034&\textbf{0.0405}$\pm$\textbf{0.0029}\\  
			solar&0.0118$\pm$0.0042&0.0106$\pm$0.0038&\textbf{0.0174}$\pm$\textbf{0.0069}&0.031$\pm$0.0084&\textbf{0.0182}$\pm$\textbf{0.0099}&\textbf{0.0137}$\pm$\textbf{0.0076}&\textbf{0.014}$\pm$\textbf{0.0054}\\  
			song&0.0917$\pm$0.0003&0.0636$\pm$0.00033&0.0706$\pm$0.00029&0.0657$\pm$0.00039&\textbf{0.0636}$\pm$\textbf{0.00052}&0.0641$\pm$0.00029&\textbf{0.0637}$\pm$\textbf{0.00052}\\  
			stock&0.0583$\pm$0.0095&0.013$\pm$0.0023&0.06$\pm$0.016&0.051$\pm$0.0088&0.53$\pm$0.41&\textbf{0.0399}$\pm$\textbf{0.014}&\textbf{0.0364}$\pm$\textbf{0.0076}\\  
			tamielectric&0.334$\pm$0.002&0.334$\pm$0.0021&\textbf{0.334}$\pm$\textbf{0.0032}&0.338$\pm$0.0028&\textbf{0.334}$\pm$\textbf{0.0033}&\textbf{0.335}$\pm$\textbf{0.0027}&\textbf{0.334}$\pm$\textbf{0.0032}\\  
			wine&0.0566$\pm$0.0028&0.0202$\pm$0.00099&0.0327$\pm$0.0031&0.0423$\pm$0.0064&\textbf{0.023}$\pm$\textbf{0.0016}&0.039$\pm$0.0023&0.0348$\pm$0.0028\\  
			yacht&0.105$\pm$0.017&0.0176$\pm$0.0055&\textbf{0.0588}$\pm$\textbf{0.024}&0.0736$\pm$0.021&0.133$\pm$0.2&0.0676$\pm$0.0096&\textbf{0.0469}$\pm$\textbf{0.018}\\  
			\hline
		\end{tabular} 
	}
\end{table}


\section{Proof of the results for \SuffP{} and  \AdaSP{}} \label{app:suffpert}
In this section, we first derive the rate for the optimization and parameter estimation error of the sufficient statistics perturbation (SuffPert) approach as was shown in Table~\ref{tab:comparison} and Table~\ref{tab:compare-efficiency}. This will build intuition towards \AdaSP{}, which we will present the proof of it towards the end of the section.  

\subsection{Analysis of \SuffP{} for linear regression}
Recall that \SuffP{} is the naive approach that uses Gaussian mechanism to release $X^TX$ and $X\vct y$ then estimate $\theta^*$ using the plug-in estimator.
\begin{lemma}\label{lem:suffpert}
	Let $\theta^* = (X^TX)^{-1}X\vct y$, and $\hat{\theta} = (X^TX + E_1)^{-1}  (X\vct y+E_2)$ for any $E_1\in\R^{d\times d}, E_2\in \R^d$ such that $X^TX + E_1$ is invertible, then
	$$
	\hat{\theta}   - \theta^*   = -(X^TX + E_1)^{-1} E_1 \theta^*  + (X^TX + E_1)^{-1} E_2. 
	$$
\end{lemma}
In \SuffP{},  $E_1$ is a symmetric Gaussian random matrix where each element in the upper triangular part of this matrix is iid $\cN(0, \frac{4\|\cX\|^4\log(4/\delta)}{\epsilon^2})$, and $E_2$ is an iid Gaussian vector drawn from $\cN(0,\frac{4\|\cX\|^2\|\cY\|^2\log(4/\delta)}{\epsilon^2})$. 

By Lemma~\ref{lem:suffpert},  and Cauchy-Schwartz, we can write
\begin{align*}
\|\hat{\theta}   - \theta^*\|^2 \leq&   2[\theta^*]^TE_1^T [(X^TX + E_1)^{-1}]^T(X^TX + E_1)^{-1} E_1 \theta^* \\&+  2E_2^T[(X^TX + E_1)^{-1}]^T(X^TX + E_1)^{-1}E_2 .
\end{align*}
This equation highlights the key artifact of this method, as when $X^TX$ has a small eigenvalue, there is a non-trivial probability that $X^TX + E_1$ will be nearly singular and that could potentially blow up the variance.

We could however analyze the high probability error bound, which becomes meaningful when $\|E_1\| <  \lambda_{\min }(X^TX) = \frac{\alpha n\|\cX\|^2}{d}$ as then we can show that with high probability, $X^TX + E_1$ has a smallest singular value that is bounded away from zero.  In particular if $\|E_1\|  \leq \lambda_{\min }(X^TX)/2$ with high probability, then we can derive an error bound using Lemma~\ref{lem:suffpert}:
$$\|\hat{\theta}   - \theta^*\|^2 = O(\frac{d^3\|\theta^*\|^2}{\alpha^2n^2\epsilon^2}) $$
under the simplifying assumption that $\|\cY\| = O(\|\cX\|\|\theta^*\|)$.

The eigenvalue condition suggests that such reasonable error bound only starts to apply when
$$n = \tilde{\Omega}(\frac{d^{1.5}\sqrt{\log(4/\delta)}}{\alpha  \epsilon}).$$

Now, using the following lemma, we can convert the optimization error into estimation in a different norm.
\begin{lemma}\label{lem:prediction2est_conversion}
	Let $\theta^* = (X^TX)^{-1}X\vct y$, for any $\theta$,
	$$
	\|\vct y - X\theta\|^2  - \|\vct y - X\theta^*\|^2  =  (\theta-\theta^*)^TX^TX (\theta-\theta^*) = \|\theta-\theta^*\|_{X^TX}^2.
	$$
\end{lemma}
\begin{proof}
	The result follows directly by the second order Taylor expansion of $	\|\vct y - X\theta\|^2 $ at $\theta^*$ and the fact that the gradient at $\theta^*$ is $0$.
\end{proof}

A direct calculation leads to the following bound
\begin{align*}
\|\hat{\theta}-\theta^* \|_{X^TX}^2 \leq& 2 [\theta^*]^T E_1^T[(X^TX + E_1)^{-1}]^{T} (X^TX)(X^TX + E_1)^{-1}  E_1\theta^* \\
&+ 2E_2^T[(X^TX + E_1)^{-1}]^{T} (X^TX)(X^TX + E_1)^{-1}E_2
\end{align*}
The idea is that by random matrix theory, we get $\|E_1\|\leq \tilde{O}(\sqrt{d} \|\cX\|^2 \sqrt{\log 12\delta}/\epsilon)$ with high probability. For large enough $n$,  $X^TX$ has a smallest eigenvalue on that order, which allows us to prove:
$$
0.5 X^TX \prec X^TX + E_1  \prec  2X^TX.
$$
with high probability.
It follows that under this high probability event
\begin{equation}\label{eq:suffp_pred_error_decomp}
\|\hat{\theta}-\theta^* \|_{X^TX}^2 \leq  8\|E_1\theta^* \|_{(X^TX)^{-1}}^2  +  8\|E_2\|_{(X^TX)^{-1}}^2.
\end{equation}

We first prove the following Johnson-Lindenstrauss type lemma for symmetric Gaussian random matrices and ellipsoid distance.
\begin{lemma}\label{lem:JL-ellipsoid}
	Let $\theta\in\R^d$ be a fixed and $E$ be a symmetric random Gaussian matrix where the upper triangular region is iid Gaussian with $\cN(0,w^2)$ With probability $1-\varrho$, and let $A$ be a positive semi-definite matrix, 
	$$\|E \theta \|_{A}^2   \leq w^2 \tr(A) \|\theta\|^2\log(2d^2/\rho)$$
\end{lemma}
\begin{proof}
	Take the eigenvalue decomposition $A = U \Lambda U^T$, we can write
	\begin{equation}\label{eq:ellipsoid_dist_preserving}
	\|E\theta\|_{A}^2 = [\theta]^T  E^T U \Lambda U^T E \theta = \sum_{i=1}^d  \lambda_{i} \sum_{j=1}^d  [U^T E]_{i,j}^2  [\theta]_i^2.
	\end{equation}
	Note $[U^T E]_{i,j} =  \sum_{k=1}^dU_{i,k} [E]_{j,k}$ where  $[E]_{j,\cdot}$ is an independent Gaussian vector, despite that $E$ itself is constrained to be a symmetric matrix. Using that $U$ is orthogonal, we have that marginally for each $i,j\in[d]^2$,
	$$[U^T E]_{i,j}  \sim \cN(0, w^2).$$
	Using the Gaussian tail bound and a union bound over all $(i,j)\in[d]^2$, we get that 
	$$
	\P(  \max_{(i,j)\in[d]^2} |[U^T E]_{i,j}| \geq  \sqrt{w^2\log(2d^2/\varrho)} )   \leq \varrho.
	$$
	Substitute this into \eqref{eq:ellipsoid_dist_preserving}, we have
	$$
	\|E\theta\|_{(X^TX)^{-1}}^2  =O\left( w^2\|\theta\|^2  \tr[A] \log(2d^2/\varrho)\right).
	$$
\end{proof}

Apply the above lemma with $A := (X^TX)^{-1}, E = E_1$ (hence $w^2 =  \frac{\log(6/\delta)\|\cX\|^4}{\epsilon^2/9}$) we get
$$
\|E_1\theta^*\|_{(X^TX)^{-1}}^2  =O\left( \frac{\|\theta^*\|^2 \|\cX\|^4 \tr[(X^TX)^{-1}] \log(6/\delta)\log(2d^2/\varrho)}{\epsilon^2/9}\right).
$$

Similarly, note that $E_2\sim U E_2$ for any unitary transformation, we can bound the tail of every eigendirection separately and that gives:
\begin{equation}
\|E_2\|_{(X^TX)^{-1}}^2 =  O\left( \frac{\|\cX\|^2\|\cY\|^2 \tr[(X^TX)^{-1}]\log(6/\delta)\log(d/\varrho)}{\epsilon^2/9}\right).
\end{equation}

Substitute the above two inequalities into \eqref{eq:suffp_pred_error_decomp}, and take union bound with the small probability event that $\|E_1\|\leq 0.5 \lambda_{\min }(X^TX)$ we get that with high probability
$$
\|\hat{\theta}-\theta^* \|_{X^TX}^2 \leq  O\left(\frac{  \|\cX\|^2(\|\cY\|^2+\|\cX\|^2\|\theta^*\|^2)   \tr[(X^TX)^{-1}]\log(6/\delta)\log(2d^2/\rho)}{\epsilon^2/9}\right).
$$
In other word, a naive \SuffP{} can perform arbitrarily poorly as  $\lambda_{\min }$ gets close to $0$.

A natural idea to address this problem is to use regularization and do ridge regression instead.  We now analyze the modified case for a fixed Ridge regression parameter $\lambda$.

\subsection{Analysis of \SuffP{} for ridge regression}
First note that \SuffP{} for ridge regression is nothing but the case when we replace $E_1$ with $\lambda I  +  E_1$.  This view allows us to reuse the lemmas we derived above. In particular, Lemma~\ref{lem:suffpert} implies the following corollary. 
\begin{corollary}
	Let $\hat{\theta}_\lambda =  (X^TX + \lambda I + E_1)^{-1} (X\vct y+E_2) $ for $\lambda >0$, then
	$$
	\hat{\theta}_\lambda   - \theta^*   = -(X^TX + \lambda I + E_1)^{-1} E_1 \theta^*  -  \lambda (X^TX + \lambda I + E_1)^{-1} \theta^*  + (X^TX +\lambda I + E_1)^{-1} E_2. 
	$$
\end{corollary}
For any psd matrix $A$
\begin{align*}
\|\hat{\theta}_\lambda   - \theta^*\|_A^2 \leq&    3\| (X^TX +\lambda I + E_1)^{-1} E_1 \theta^*\|_A^2  \\&+  3\|(X^TX +\lambda I+ E_1)^{-1}\|_A^2 \\
&+  3\lambda^2 \|(X^TX +\lambda I + E_1)^{-1}\theta^*\|_A^2.
\end{align*}
Under the high probability event such that $\|E_1\|  \leq (\lambda_{\min }(X^TX) + \lambda) /2$, we have $X^TX + \lambda I + E_1\prec 0.5(X^TX + \lambda I )$  and it implies that
$$
\|\hat{\theta}_\lambda   - \theta^*\|^2  =  O\left(\|E_1\theta^*\|_{(X^TX + \lambda I)^{-2}}^2 +   \|E_2\|_{(X^TX + \lambda I)^{-2}}^2  +  \lambda^2 \|\theta^*\|_{(X^TX + \lambda I)^{-2}}^2\right).
$$
Similarly by Lemma~\ref{lem:prediction2est_conversion}
\begin{align*}
F(\hat{\theta}_\lambda) - F(\theta^*) &= \frac{1}{2}\|\hat{\theta}_\lambda   - \theta^*\|_{(X^TX)^{-1}}^2\\
&= O\left(   \|E_1\theta^*\|^2_{(X^TX+\lambda I)^{-1}}   +  \|E_2\|^2_{(X^TX+\lambda I)^{-1}}  +  \lambda^2 \|\theta^*\|^2_{(X^TX+\lambda I)^{-1}}\right).
\end{align*}
Apply the distance preserving results in Lemma~\ref{lem:JL-ellipsoid} to the first term above with $A = (X^TX + \lambda I)^{-2}$ and $A=(X^TX+\lambda I)^{-1}$ respectively, we can write
\begin{equation}\label{eq:ssp_ridge}
\|\hat{\theta}_\lambda   - \theta^*\|^2 =  O\left(   \frac{\|\cX\|^2(\|\cY\|^2+\|\cX\|^2\|\theta^*\|^2)\log(6/\delta)\log(2d^2/\rho)}{\epsilon^2}  \tr[(X^TX+\lambda I)^{-2}] +     \lambda^2\|\theta^*\|_{(X^TX+\lambda I)^{-2}}^2 \right),
\end{equation}
\begin{align}
F(\hat{\theta}_\lambda) - F(\theta^*) &=  O\left(   \frac{\|\cX\|^2(\|\cY\|^2+\|\cX\|^2\|\theta^*\|^2)\log(6/\delta)\log(2d^2/\rho)}{\epsilon^2}  \tr[(X^TX+\lambda I)^{-1}] +     \lambda^2\|\theta^*\|_{(X^TX+\lambda I)^{-1}}^2 \right)\nonumber\\
&= O\left(   \frac{d \|\cX\|^2(\|\cY\|^2+\|\cX\|^2\|\theta^*\|^2)\log(6/\delta)\log(2d^2/\rho)}{(\lambda + \lambda_{\min})\epsilon^2 }  +  \lambda\|\theta^*\|^2 \right). \label{eq:ssp_ridge_prediction}
\end{align}
Note that when $\lambda_{\min } = 0$, choosing 
\begin{equation}\label{eq:oracle_choice_lambda}
\lambda  = \Theta\left(\sqrt{d \log(6/\delta)\log(2d^2/\rho) }\Big(\frac{\|\cX\|\|\cY\|}{\|\theta^*\|} + \|\cX\|^2\Big) / \epsilon\right)
\end{equation}
balances the two terms and results in a bound that is on the order of $\sqrt{d}/\epsilon$ which matches the lower bound for the Lipschitz private ERM \eqref{eq:dp_lowerbound_lipschitz}. Similarly, when $\lambda_{\min }$ is larger than the above quantity, the optimal choice of $\lambda$ is $0$ and we get a rate of $ d/(\lambda_{\min }\epsilon^2)$ which matches the lower bound for the private strongly convex ERM \eqref{eq:dp_lowerbound_strongcvx}.

It remains to check whether such choices are feasible, because recall that the entire analysis hinges upon the event that 
\begin{equation}\label{eq:cond_lambda_critical}
\|E_1\|  \leq (\lambda_{\min }  + \lambda) /2.
\end{equation}
Recall that $\|E_1\|  \leq 2\sqrt{d \log(6/\delta)\log(d^2/\rho)}\|\cX\|^2/ (\epsilon/3)$
with high probability, so the a choice of $\lambda$  that satisfies \eqref{eq:oracle_choice_lambda} with appropriate constant automatically obeys \eqref{eq:cond_lambda_critical}.

\subsection{Analysis of \AdaSP{}. Proof of Theorem~\ref{thm:adaSuffP}. }
The proof of Statement (i) is a  straightforward application of the composition theorem over standard releases of $X^TX$, $X^T\mathbf{y}$ and $\lambda_{\min }(X^TX)$.

The extension from \SuffP{} to \AdaSP{} involves choosing $\lambda$ adaptively. By our analysis above, the desired choice is \eqref{eq:oracle_choice_lambda} but it depends on unknown quantities of the data $\lambda_{\min }$ and $\|\theta^*\|$. 

Our choice of $\lambda$ in Algorithm~\ref{alg:adaSuffP} is that 
$$
\lambda = \max\left\{0 , \frac{\sqrt{d \log(6/\delta)\log(2d^2/\rho) }\|\cX\|^2}{\epsilon/3}  -  \lambda_{\min}^*. \right\}
$$
where $\lambda_{\min }^*$ is a differentially private high probability lower bound of $\lambda_{\min }$.
Check that the choice obeys \eqref{eq:cond_lambda_critical} so the error analysis above is valid. 
Substitute this choice of $\lambda$ into \eqref{eq:ssp_ridge_prediction} and \eqref{eq:ssp_ridge}, we get the results in Theorem~\ref{thm:adaSuffP}(ii) and Theorem~\ref{thm:adaSuffP}(iii). 

Note that because we do not know $\|\theta^*\|$, we cannot set the $\frac{\|\cX\|\|\cY\|}{\|\theta^*\|}$ as part of the oracle $\lambda$ choice in \eqref{eq:oracle_choice_lambda}. As a result, the final optimization error is proportional to the constant of
$
\|\cY\|^2 + \|X\|^2\|\theta^*\|^2
$
instead of the optimal
$
\|\cY\|\|\cX\|\|\theta^*\|  +  \|X\|^2\|\theta^*\|^2.
$
They are on the same order under our assumption that $\|\cY\| \asymp \|\cX\| \|\theta^*\|$.

\section{Proofs related to \AdaOPS{}}		\label{sec:proof_OPS}

The proof uses the pDP technique that first analyzes $\OPS{}$ for fixed set of tuning parameters and then do pDP to DP conversion with differential privately chosen tuning parameters.

\subsection{Utility of \OPS{} with fixed $(\gamma,\lambda)$}
\begin{lemma}[Parameter estimation error]\label{lem:estimation-err}
	Let $X$ be fixed and $\theta^*_\lambda$ be the maximum a posteriori estimator (MLE if $\lambda =0$), and $\tilde{\theta}$ be the output of \OPS{} with parameter $\gamma,\lambda$, then:
	
	\begin{enumerate}
		\item for all $0<\varrho<1$, with probability $1-\varrho$ 
		$$
		\|\tilde{\theta} - \theta^*_\lambda\|_{X^TX + \lambda I}^2  \leq \frac{d + 2\sqrt{d\log(1/\varrho)} +2\log(1/\varrho)}{\gamma} \leq \frac{5d\log(1/\varrho)}{\gamma}.
		$$
		\item It holds that
		$$
		\E[\tilde{\theta}|X,\vct y] = \theta^*_\lambda \text{ and } \Cov[\tilde{\theta} |  X,\vct y ]  =  \gamma^{-1} (X^TX+\lambda I)^{-1}.
		$$
		\item If we assume that $\vct y \sim \cN(X\theta_0, \sigma^2 I)$ and $\lambda =0$, then
		$$
		\E[\tilde{\theta}|X] = \theta_0 \text{ and } \Cov[\tilde{\theta} |  X ]  =  (\sigma^2 +  \gamma^{-1}) (X^TX)^{-1}.
		$$
	\end{enumerate} 
\end{lemma}
\begin{proof}
	Let $H := X^TX + \lambda I  =  Q\Lambda Q^T$, and let $Z\sim \cN(0,\gamma^{-1} I_d)$
	$$
	\|\tilde{\theta} - \theta^*_\lambda\|_{H}^2  = (\tilde{\theta} - \theta^*_\lambda)^TH (\tilde{\theta} - \theta^*_\lambda) = Z^T \Lambda^{-1/2} Q^T Q \Lambda Q^T Q\Lambda^{-1/2} Z  = \|Z\|_{2}^2
	$$
	Note that $\gamma \|Z\|_{2}^2$ has a $\chi^2$-distribution with degree of freedom $d$, by the standard right tail bound inequality of $\chi^2$ R.V., we get the results as claimed.
	The second statement is trivial and it follows directly from the algorithm.
	For the third statement, note that the MLE $\theta^*$ is unbiased for linear regression, also, it has covariance matrix $\sigma^2(X^TX)^{-1}$. The second part of the randomness comes from sampling from the posterior distribution which has covariance matrix $\gamma^{-1} (X^TX)^{-1}$ by the algorithm. The results follows after noting that the we are adding independent noise.
\end{proof}

\begin{lemma}[Optimization error / regret bound]
	Let $\theta^*$ be a local minimum of a convex quadratic function $F$  and
	$$
	\tilde{\theta} \sim \cN(\theta^*, \gamma^{-1}[\nabla^2 F(\theta^*)]^{-1}),
	$$
	then for all $0<\varrho<1$, with probability $1 - \varrho$ 
	$$
	F(\tilde{\theta}) - F(\theta^*)  \leq \frac{d + 2\sqrt{d\log(1/\varrho)} + 2\log(1/\varrho)}{2\gamma}  \leq  \frac{2.5 d\log(1/\varrho)}{\gamma}.
	$$
\end{lemma}
\begin{proof}
	Since $F$ is quadratic, $\nabla^2 F \equiv H$ for some fixed matrix $H$ (independent to location). By Taylor's theorem
	$$
	F(\tilde{\theta}) - F(\theta^*)  =  \langle \nabla F(\theta^*),  \tilde{\theta}-\theta^*\rangle   +   \frac{1}{2}\|\tilde{\theta}-\theta^*\|_{H}^2.
	$$
	Substitute Lemma~\ref{lem:estimation-err} into the above we get the result as claimed.
\end{proof}

\subsection{pDP analysis of \OPS{} for fixed $(\gamma,\lambda)$}
We now cite the per-instance differential privacy of \OPS{} for  a fixed set of parameters from \citep{wang2017per}.
\begin{theorem}[Theorem~15 of \citet{wang2017per} ]\label{thm:pdp_analysis}
	Consider the algorithm that samples from 
	$$
	p(\theta|X,\vct y) \propto e^{-\frac{\gamma}{2} \left(\|\vct y - X\theta\|^2 + \lambda\|\theta\|^2\right)}.
	$$
	Let $\hat{\theta}$ and $\hat{\theta}'$ be the ridge regression estimate with data set $X\times \vct y$ and $[X,x]\times [\vct y,y]$ and defined the out of sample leverage score
	$\mu := x^T(X^TX + \lambda I)^{-1}x = x^TH^{-1}x$ and in-sample leverage score $\mu' := x^T [(X')^TX' + \lambda I]^{-1}x = x^T(H')^{-1}x $.
	Then for every $\delta>0$, privacy target $(x,y)$, the algorithm is $(\epsilon,\delta)$-pDP with
	\begin{align}
	&\epsilon(Z,z) \leq 	\frac{1}{2}\left| -\log(1+\mu) + \frac{\gamma\mu}{(1+\mu)}(y-x^T\hat{\theta})^2\right| + \frac{\mu}{2} \log (2/\delta) + \sqrt{\gamma\mu \log (2/\delta)} |y-x^T\hat{\theta}| \label{eq:eps_master1}\\
	=&\frac{1}{2}\left|-\log(1-\mu') - \frac{\gamma\mu'}{1-\mu'}(y-x^T\hat{\theta}')^2\right| + \frac{\mu'}{2} \log(2/\delta) + \sqrt{\gamma\mu'\log(2/\delta)}|y-x^T\hat{\theta}'|.\label{eq:eps_master2}
	\end{align}
\end{theorem}
\begin{remark}\label{rmk:pDP_eps}
	Let $L:=\|\cX\|(\|\cX\|\|\theta^*_\lambda\| + \|\cY\|)$, 
	The \OPS{} algorithm for ridge regression with parameter $(\lambda,\gamma)$ obeys $(\epsilon,\delta)$-pDP for each data set $(X,y)$ and all target $(x,y)$ with
	$$
	\epsilon = \sqrt{ \frac{\gamma L^2 \log(2/\delta)}{\lambda + \lambda_{\min }}} +  \frac{\gamma L^2}{2 (\lambda + \lambda_{\min }+\|\cX\|^2)}  + \frac{ (1+\log(2/\delta))\|\cX\|^2}{2(\lambda + \lambda_{\min })}.
	$$
\end{remark}

\subsection{pDP to DP conversion}

The hallmark of DP algorithm design is that one needs to calibrate the amount of noise so that no matter what data set is sent into the algorithm, the algorithm meets a prescribed privacy budget $(\epsilon,\delta)$. The pDP guarantees of \OPS{} says that for a fixed randomized algorithm, if the data set is nice, then the privacy guarantee is strong, while if the data set is poorly-conditioned, then the privacy loss is big. What is more, the pDP analysis illustrates that the key ingradients of that appears in the pDP bound is the smallest eigenvalue of $X^TX$ and the local Lipschitz constant $L$ (given as a function of $\|\cX\|$ and $\|\cY\|$ and the magnitude of the solution $\|\theta^*_\lambda\|$). 


The approach used in \citet{wang2017per} is to differentially privately release $\lambda_{\min }$ and an adaptive amount of regularization $\lambda$ is added so that a pre-specified strong convexity parameter $\alpha^*$ is met with high probability. Then a crude upper bound of $\|\theta^*_\lambda\|$ is used based on $\lambda^*$ or $\lambda_{\min }$ (if larger than $\lambda^*$) to calibrate $\gamma$. The outcome is an asymptotically efficient differentially private estimator of linear regression coefficients when the data set is well-conditioned. However, there are two issues.
First, it is unclear how $\lambda^*$ is chosen; second, the crude upper bound of $\|\theta^*_\lambda\|$ leads to unnecessary dimension dependence in the bound. 

In this section, we further extend the idea by proposing a novel way of releasing the $\|\theta^*_\lambda\|$ differential privately by injecting a multiplicative noise, which allows us to design a DP algorithm that adapts to small local Lipschitz constant near the optimal solution and also a principled approach of choosing the regularization parameter $\lambda$,
such that (1) the algorithm is $(\epsilon,\delta)$-DP for all input data, (2) it is statistically efficient with an improved dimension-dependence when the data follows a linear Gaussian model (3) the optimization error is optimal up to a logarithmic term for each (unknown) strong convexity parameter and local Lipschitz constant separately.

The algorithm basically looks like the following:
\begin{enumerate}
	\item Differentially privately release $\lambda_{\min }$ using $(\epsilon/4,\delta/3)$, and choose regularization parameter $\lambda$ accordingly.
	\item Condition on a high probability event of $\lambda_{\min }$, and choose $\lambda$.
	\item Differentially privately release $\|\theta^*_{\lambda}\|$ using $(\epsilon/4,\delta/3)$, where $\theta^*_{\lambda} =  (X^TX + \lambda I)^{-1} X^T\vct y$.
	\item Condition on a high probability event of both $\lambda_{\min }$ and $\|\theta^*_{\lambda}\|$, calibrate the noise to meet the $(\epsilon/2,\delta/3)$ requirement.
\end{enumerate}

We start by showing how we can release $\lambda_{\min }$ and $\|\theta^*_{\lambda}\|$. By Weyl's lemma, $\lambda_{\min }$ has a global sensitivity of $\|\cX\|^2$. It turns out that while $\|\theta^*_{\lambda}\|$ does not have a well-behaved global or local sensitivity, a logarithmic transformation $ 	\log(\|\cY\|+\|\cX\|\|\theta^*_{\lambda}\|) $ has a very stable local sensitivity that is parameterized only by the smallest eigenvalue, which we can easily construct a differentially private upper bound.
\begin{lemma}\label{lem:lipschitz_GS}
	Let $\theta^*_\lambda$ be the ridge regression estimate with parameter $\lambda$ and the smallest eigenvalue of $X^TX$ be $\lambda_{\min}$, then the function
	$
	\log(\|\cY\|+\|\cX\|\|\theta^*_\lambda\|)  
	$
	has a local sensitivity of $ \log(1 +  \frac{\|\cX\|^2}{\lambda_{\min } + \lambda}) $.
\end{lemma}
\begin{proof}
	Denote $\|\cY\| =: \alpha$ and $\|\cX\| =: \beta$.  Let the data point being added to the data set be $(x,y)$. For a fixed $\lambda$, denote $\hat{\theta}$ and $\hat{\theta}'$ as the ridge regression estimate with parameter $\lambda$ on data set $X,\mathbf{y}$  and  $[X,x],[\mathbf{y},y] $ respectively.
	
	By Lemma~\ref{lem:smoothlearning}, we have
	$$\left|  \|\hat{\theta}\| - \|\hat{\theta}'\| \right|   \leq \|\hat{\theta} - \hat{\theta}'\|  =  | y- x^T\hat{\theta}| \sqrt{x^T ([X,x]^T[X,x] + \lambda I)^{-2} x} \leq  \frac{\beta}{\lambda_{\min } + \lambda} (\alpha+  \beta \min\{ \|\hat{\theta}'\|,\|\hat{\theta}\|\}).$$
	Multiplying $\beta$ on both sides and use triangular inequality, we have
	$$
	\begin{cases}
	(\alpha+ \beta\|\hat{\theta}\|) -  (\alpha+\beta\|\hat{\theta}'\|)   \leq   \frac{\beta^2}{\lambda_{\min } + \lambda}   (\alpha + \beta \|\hat{\theta}'\|)\\
	(
	(\alpha +\beta \|\hat{\theta}'\|)  - (\alpha + \beta \|\hat{\theta}\| ) \leq  \frac{\beta^2}{\lambda_{\min } + \lambda}   (\alpha+ \beta \|\hat{\theta}\|)
	\end{cases}
	$$
	Rearrange the terms and take log on both sides, we get
	$$
	\left| \log \frac{\alpha+ \beta\|\hat{\theta}\|}{\alpha+ \beta\|\hat{\theta}'\|} \right|  \leq   \log(1 +  \frac{\beta^2}{\lambda_{\min } + \lambda}   ).
	$$
\end{proof}

\subsection{Automatically choosing $\lambda$}\label{sec:adaptive_choice_of_lambda}
We will do this by minimizing an upper bound of the empirical risk.
Note that this is a somewhat circular problem because the empirical risk is a function of $\|\theta^*_\lambda\|$, but in order to release it differential privately, we need to choose $\lambda$ to begin with. The main idea is to express the DP upper bound of the Lipschitz constant analytically as a function of $\lambda$ and also take the additional noise from differential privacy into account.

Let a differentially private lower bound of $\lambda_{\min }$ be $\tilde{\lambda}_{\min }$, and $\bar{L}$ be a high-probability upper bound of the local Lipschitz constant  $L = \|\cX\|(\|\cY\|+\|\cX\|\|\theta^*_\lambda\|)$. Consider the fixed OPS algorithm with the parameter choice of $\gamma^{-1}  =  \frac{\log(2/\delta) L^2}{(\tilde{\lambda}_{\min } + \lambda)\epsilon^2}$. Define 
\begin{align}
C_1(\epsilon,\delta,\varrho,d) &:=  \frac{[d/2 + \sqrt{d\log(1/\varrho)} + \log(1/\varrho)] \log(2/\delta)}{\epsilon^2}
\label{eq:C1}\\ 
C_2(\epsilon,\delta) &:=  \frac{\log(2/\delta)}{\epsilon} \label{eq:C2}
\end{align}

We know from Lemma~\ref{lem:lipschitz_GS}, that we can construct a high probability upper bound $L$ from a differentially private release of $\log(\|\cY\|+ \|\theta^*_\lambda\|\|\cX\|)$ satisfying that with probability $1-\delta$
$$\|\cX\|(\|\cY\|+ \|\cX\|\|\theta^*_\lambda\|):= L  \leq \bar{L}  \leq L (1+\|\cX\|^2/(\tilde{\lambda}_{\min } + \lambda))^{C_2} $$ 


Recall that $\theta^*$ is the least square solution $\argmin_{\theta} F(\theta)$. The optimization error obeys that
\begin{align*}
F(\tilde{\theta}) -  F(\theta^*)  =& F(\tilde{\theta}) + \lambda \|\tilde{\theta}\|^2 -  F(\theta^*_\lambda) - \lambda \|\theta^*_\lambda\|^2 \\
&+ F(\theta^*_\lambda) + \lambda \|\theta^*_\lambda\|^2  -  F(\theta^*)  -\lambda \|\theta^*\|^2  + \lambda \|\theta^*\|^2 \\
\leq& \frac{[d/2 + \sqrt{d\log(1/\varrho)} + \log(1/\varrho)]}{\gamma}   + 0  + \lambda \|\theta^*\|^2\\
=& \frac{C_1(\epsilon,\delta,\varrho,d)  \|\cX\|^2(\|\cY\|+ \|\cX\|\|\theta^*_\lambda\|)^2 (1+\frac{\|\cX\|^2}{ \lambda + \tilde{\lambda}_{\min } })^{2 C_2(\epsilon,\delta)}}{ \lambda + \tilde{\lambda}_{\min }  }  +  \lambda \|\theta^*\|^2\\
\leq& \frac{C_1(\epsilon,\delta,\varrho,d)  \|\cX\|^4(\|\cY\|/\|\cX\|+ \|\theta^*\|)^2 (1+\frac{\|\cX\|^2}{ \lambda + \tilde{\lambda}_{\min } })^{2 C_2(\epsilon,\delta)}}{ \lambda + \tilde{\lambda}_{\min }  }  +  \lambda (\|\cY\|/\|\cX\|+\|\theta^*\|)^2
\end{align*}
The first inequality uses Lemma~\ref{lem:prediction2est_conversion}, Lemma~\ref{lem:estimation-err} and used the optimality of $\theta^*_\lambda$ for the regularized objective.
In the last line, we used the monotonicity of ridge regression which says that for all $\lambda >0$, we have
$
\|\theta^*_\lambda\| \leq \|\theta^*\|.
$
We also $\|\theta^*\|$ into $\|\cY\|/\|\cX\| + \|\theta^*\|$.

This relaxation allows us to choose $\lambda$ that is independent to $\|\theta^*\|$, by minimizing the second part of the upper bound
\begin{equation}\label{eq:opterr_upperbound}
F(\tilde{\theta}) -  F(\theta^*)  \leq   (\|\cY\|/\|\cX\| + \|\theta^*\|)^2  \left[  \frac{C_1(\epsilon,\delta,\varrho,d) \|\cX\|^4(1+\frac{\|\cX\|^2}{ \lambda + \tilde{\lambda}_{\min } })^{2 C_2(\epsilon,\delta)}}{ \lambda + \tilde{\lambda}_{\min }  } + \lambda \right].
\end{equation}
The only thing that we need to privately release to choose $\lambda$ is $\tilde{\lambda}_{\min }$ which has a fixed global sensitivity. 
The detailed procedure was summarized in Algorithm~\ref{alg:self-tuning-AdaOPS}.

\subsection{Proof of Theorem~\ref{thm:adaops}}
We will now formally prove the theoretical guarantees of \AdaOPS{} that we stated in Theorem~\ref{thm:adaops}.
\begin{proof}[Proof of Theorem~\ref{thm:adaops} (i)]
	First of all, $\tilde{\lambda}_{\min }$ has global sensitivity $\|\cX\|^2$ by Weyl's lemma (Lemma~\ref{lem:weyl}). Using Gaussian mechanism, $\tilde{\lambda}_{\min }$ is an $(\epsilon/4,\delta/3)$-DP release.  Now, by the standard Gaussian tail bound, under the same probability event that holds with probability $1-\delta/3$, we know that 
	$$\left| \tilde{\lambda}_{\min }  - \lambda_{\min }\right|  \leq  \frac{\log(6/\delta)}{\epsilon/4}.$$
	Condition on this event, and apply Lemma~\ref{lem:lipschitz_GS}, we know $\log(\|\cY\|+\|\cX\|\|\theta^*_\lambda\|)$ has (conditional) global sensitivity of $\log(1 + \frac{\|\cX\|^2}{\tilde{\lambda}_{\min }+ \lambda})$.  So for any choice of $\lambda$  (that uses only privately released information), the algorithm release $\log(\|\cY\|+\|\cX\|\|\theta^*_\lambda\|)$ using Gaussian mechanism. Again by Gaussian tail bound, we know that $\Delta$ (in Algorithm~\ref{thm:adaops}) is a high probability upper bound  of $\log(\|\cY\|+\|\cX\|\|\theta^*_\lambda\|)$ and the event is the same as the event of success in this $(\epsilon/4,\delta/3)$-DP. In other word, we have that conditioning on the event with probability $1-2\delta/3$, for any data set $(X,\vct y)$ and any target $(x,y)$.
	$$
	|y -  x^T\theta^*_\lambda|  \leq \|\cY\|+\|\cX\|\|\theta^*_\lambda\|  \leq e^\Delta  \leq (\|\cY\|+\|\cX\|\|\theta^*_\lambda\|)  (1 + \frac{\|\cX\|^2}{\lambda + \tilde{\lambda}_{\min }})^{\frac{\log(6/\delta)}{\epsilon/4}}.
	$$
	Denote $L:=\|\cX\|e^\Delta$ and choose $\gamma$ according to the Step 7 of the algorithm block. Condition on the high probability event on the eigenvalue and and local Lipschitz constant, the results in Theorem~\ref{thm:pdp_analysis} (and the remark underneath it) implies that $\tilde{\theta}$ is an $(\epsilon/2,\delta/3)$-pDP for all pairs of adjacent data sets that differs by adding or removing one data point. This by definition implies that we have converted the pDP guarantee to $(\epsilon/2,\delta/3)$-DP.
	
	Finally, by the adaptive simple composition of the three DP mechanisms, we conclude that \AdaOPS{} is $(\epsilon,\delta)$-DP.
\end{proof}

We now move on the  analyze the utility of \AdaOPS{} in terms of optimization error and estimation error (Statement (ii) and (iii) in Theorem~\ref{thm:adaops}). 

\begin{proof}[Proof of Theorem~\ref{thm:adaops} (ii)]
	The key idea of the proof is to establish that the way $\lambda$ is chosen according to \eqref{eq:adachoice_of_lambda} is effectively minimizing an upper bound of the optimization error, according to our derivation to that leads to \eqref{eq:opterr_upperbound}.
	To start, note that $C_1$ and $C_2$ in \eqref{eq:adachoice_of_lambda} are both positive for any parameters that are passed into them, so the first term in the upper bound is monotonically decreasing in $\lambda$ and the second term is monotonically increasing in $\lambda$ so there is a unique $\lambda$ minimizing the criterion.
	
	Let $\tilde{C}_1$ be an arbitrary upper bound of $C_1$. Also recall that $t_{\min} = \max\left\{\frac{\|\cX\|^2(1+\log(6/\delta))}{2\epsilon}-\tilde{\lambda}_{\min},0\right\}$
	\begin{align}
	\min_{t\geq t_{\min}}\frac{C_1\|\cX\|^4(1+ \frac{\|\cX\|^2}{t +\tilde{\lambda}_{\min }})^{2C_2}}{t + \tilde{\lambda}_{\min }} + t  &\leq \min_{t \geq \max\{t_{\min},\|\cX\|^2C_2 - \tilde{\lambda}_{\min }\} } \frac{\tilde{C}_1\|\cX\|^4(1+ \frac{1}{\lambda +\tilde{\lambda}_{\min }})^{2C_2}}{t + \tilde{\lambda}_{\min }} + t \nonumber\\
	&\leq \min_{t\geq \max\{t_{\min },\|\cX\|^2C_2 - \tilde{\lambda}_{\min }\} }\frac{e^2\tilde{C}_1 \|\cX\|^4}{t + \tilde{\lambda}_{\min }} +  t\nonumber\\
	&\lesssim \min\left\{\frac{e^2\tilde{C}_1\|\cX\|^4}{\tilde{\lambda}_{\min }}, e\|\cX\|^2\sqrt{\tilde{C}_1}\right\}.\label{eq:opterr_upperbound2}
	\end{align}
	The first inequality follows because we are increasing $C_1$ and also restricting the domain we optimize over, the second inequality uses that $(1+1/x)^x\leq e$ for all $x>0$. 
	
	The third inequality is true when
	$$
	e\sqrt{\tilde{C}_1} \geq \max\left\{C_2,\frac{1+\log(6/\delta)}{2\epsilon}\right\}  .
	$$
	To check this, discuss two cases of $\tilde{\lambda}$. In the first case, if $\tilde{\lambda}_{\min } \leq e\sqrt{\tilde{C_1}}\|\cX\|^2$ we can take the feasible  $t =  e\|\cX\|^2 \sqrt{\tilde{C_1}}$  and obtain the second expression. In the second case, we know that taking $t=0$ is feasible, which gives rise to the first bound.
	
	Take $\tilde{C}_1   =  C_1 \vee  e^{-2} C_2^2  \vee e^{-2} \frac{(1+\log(6/\delta))^2}{4\epsilon^2}$.
	
	%
	
	We now look closer into parameters in $C_1$ and $C_2$ of \eqref{eq:adachoice_of_lambda}.
	
	First of all, since $\tilde{\epsilon}< \epsilon/2$, 
	$$\bar{\epsilon}  =  \epsilon/2  -  \frac{\epsilon^2}{8}\left[\frac{1}{\log(6/\delta)} +  \frac{1+\log(6/\delta)}{\log(6/\delta)}\right]  \leq \epsilon/2  -  \frac{\tilde{\epsilon}^2}{2}\left[\frac{1}{\log(6/\delta)} +  \frac{1+\log(6/\delta)}{\log(6/\delta)}\right] \leq \tilde{\epsilon}.$$
	This implies that $\bar{\epsilon}<\epsilon/2$. On the other hand, by the assumption that $\epsilon<2\log(6/\delta) / (1+\log(6/\delta))$, 
	$\bar{\epsilon}  > \epsilon/2  - \epsilon/4  = \epsilon/4$.
	
	It follows that 
	\begin{align*}
	\tilde{C}_1 &=   \max\{C_1(\bar{\epsilon},\delta/3,\varrho,d), C_2(\epsilon/4,\delta/3)^2 e^{-2} ,  e^{-2} \frac{(1+\log(6/\delta))^2}{4\epsilon^2}  \}  \\
	&=  \max\left\{\frac{[d/2 + \sqrt{d \log(1/\varrho)} + \log{1/\varrho}]  \log(6/\delta)\} }{\bar{\epsilon}^2},  \frac{16\log(6/\delta)^2}{e^2 \epsilon^2},  \frac{(1+\log(6/\delta))^2}{4\epsilon^2\epsilon^2}\right\}\\
	&\leq \frac{16 \log(6/\delta) \max\{[d/2 + \sqrt{d \log(1/\varrho)} + \log{1/\varrho}],\frac{\log(6/\delta)}{e^2} \}  }{\epsilon^2} \\
	&= O\left(\frac{\max\{d,\log(1/\varrho)\} \log(6/\delta) +\log^2(6/\delta)}{\epsilon^2}\right).
	\end{align*}	
	
	Apply the above upper bound to \eqref{eq:opterr_upperbound2} and then to \eqref{eq:opterr_upperbound}, we get that with probability $1-2\delta/3 - \varrho$, then simultaneously,
	\begin{align*}
	F(\tilde{\theta}) - F(\theta^*)  &\leq O\left((\|\cY\|^2 + \|\cX\|^2\|\theta^*\|^2)\|\cX\|^2   \frac{(d+\log(1/\varrho))\log(6/\delta) +\log^2(6/\delta)}{\epsilon^2\tilde{\lambda}_{\min }} \right),\\
	F(\tilde{\theta}) - F(\theta^*)  &\leq O\left((\|\cY\|^2 + \|\cX\|^2\|\theta^*\|^2)  \frac{\sqrt{(d+\log(1/\varrho))\log(6/\delta) + \log^2(6/\delta)}}{\epsilon}\right).
	\end{align*}
	The first bound is the smaller of the two only when $ \tilde{\lambda}_{\min }   \geq  \sqrt{\tilde{C}_1} $ and in such cases 
	$$
	\frac{1}{\lambda_{\min }}  \leq \frac{1}{\tilde{\lambda}_{\min } - \|\cX\|^2\frac{\log(6/\delta)}{\epsilon/4} }  = O(\frac{1}{ \|\cX\|^2\sqrt{\tilde{C}_1}}).
	$$
	The proof is complete by converting $\lambda_{\min }$ into the alternative normalized representation with $\alpha =  \frac{d\lambda_{\min }}{n \|\cX\|^2}$.
\end{proof}

It remains to prove the results about the estimation error under the linear Gaussian model. 

\begin{proof}[Proof of Theorem~\ref{thm:adaops} (iii)]
	Note that $\lambda_{\min }  \geq \alpha n/d$.  As we've seen in the proof of Statement (ii), with probability $1-\delta/3$, 
	$$\tilde{\lambda}_{\min } \geq \lambda_{\min } - \|\cX\|^2\frac{4\log(6/\delta)}{\epsilon}.$$ 
	Let this be event $E$. Event $E$ ensures that (under the stated assumption on $\epsilon,\delta$) we have $\frac{4\log(6/\delta)}{\epsilon}<  \frac{\alpha n}{2d}$, this implies that $\tilde{\lambda}_{\min } > 0.5 \lambda_{\min }$ and in addition, the automatic choice of $\lambda$ using \eqref{eq:adachoice_of_lambda} will be $\lambda = 0$.
	
	$$
	\E(\tilde{\theta}| X,E)  =  \E\left[ \E\big[ \tilde{\theta} \big| X,E, \vct y, L, \tilde{\lambda} \big]  \middle| X,E\right] \explain{=}{\text{Lemma~\ref{lem:estimation-err}}}  \E\left[ \theta^*_\lambda  \middle| X,E\right] \explain{=}{\lambda=0\text{ under }E}\E\left[ \theta^*  \middle| X,E\right] = \theta_0
	$$
	\begin{align}
	\Cov(\tilde{\theta} | X,E)  &=  \E\left[   \Cov\big[ \tilde{\theta}  \big|  X,E, \vct y, L, \tilde{\lambda}  \big]\middle| X,E \right] +  \Cov\left[ \E\big[ \tilde{\theta} \big|  X,E, \vct y, L, \tilde{\lambda} \big] \middle|  X,E\right] \nonumber\\
	&=  \E\left[   \frac{L^2 \log(6/\delta)}{\tilde{\lambda}_{\min }\tilde{\epsilon}^2}   \middle| X,E \right]  (X^TX)^{-1}   +  \Cov\left[ \theta^* \middle|  X,E\right] \nonumber\\
	&\prec   \frac{\E\left[ L^2 \middle| X,E \right] \log(6/\delta)}{(\lambda_{\min }/2)   (\epsilon/4)^2}    (X^TX)^{-1} +  \sigma^2(X^TX)^{-1}\label{eq:cov_calc}
	\end{align}
	
	Plugging in the consequence of $E$ into the expression of $L$ in the Algorithm~\ref{alg:self-tuning-AdaOPS}, we have
	\begin{align*}
	L  &= \|\cX\|e^\Delta =  e^{\log(\|\cY\|+\|\cX\|\|\theta^* \|) + \frac{\log(1+\|\cX\|^2/\tilde{\lambda}_{\min })}{\epsilon/4}\sqrt{\log(6/\delta)} Z +  \frac{\log(1+\|\cX\|^2/\tilde{\lambda}_{\min })}{\epsilon/4}\log(6/\delta)}\\
	&=   \|\cX\|(\|\cY\|+\|\cX\|\|\theta^* \|) (1+\|\cX\|^2/\tilde{\lambda}_{\min })^{\frac{\log(6/\delta)}{\epsilon/4}}  (1+\|\cX\|^2/\tilde{\lambda}_{\min })^{\frac{\sqrt{\log(6/\delta)}}{\epsilon/4} Z} \\
	&\leq  \|\cX\|(\|\cY\|+\|\cX\|\|\theta^* \|)    e^{ \frac{\|\cX\|^2\log(6/\delta)}{\lambda_{\min }\epsilon/8}}  e^{\frac{\|\cX\|^2\sqrt{\log(6/\delta)}}{\lambda_{\min }\epsilon/8} Z}\\
	&\leq \|\cX\|(\|\cY\|+\|\cX\|\|\theta^* \|)  e  \cdot e^{Z} 
	\end{align*}
	Take expectation of $L^2$, use the independence of $Z$ and $\vct y$  we get
	\begin{align*}
	\E[L^2 | X,E]  &=   \E[\|\cX\|^2(\|\cY\|+\|\cX\|\|\theta^* \|)^2 ] e^2  \E[e^{2Z}] \\
	&=  e^4   (2\|\cX\|^2\|\cY\|^2 +  2\|\cX\|^4\E[\|\theta^*\|^2| X])\\
	& = 2e^4\|\cX\|^2(\|\cY\|^2 + \|\cX\|^2\|\theta_0\|^2 + \sigma^2\|\cX\|^2\tr[(X^TX)^{-1}])
	\end{align*}
	where in the second line, we used the formula for the moment generating function of standard normal distribution, and then in the third line, we used that $\theta^* \sim \cN(\theta_0, \sigma^2(X^TX)^{-1})$.
	Substitute into \eqref{eq:cov_calc} and replace $\lambda_{\min }$ with $n\|\cX\|^2\alpha/d$ we get 
	$$
	\Cov(\tilde{\theta} | X,E)  \prec \left(1 +  \frac{64e^4\Big[\|\cY\|^2 + \|\cX\|^2\|\theta_0\|^2 + \sigma^2\|\cX\|^2\tr[(X^TX)^{-1}]\Big] d\log(6/\delta) }{\alpha \sigma^2  n \epsilon^2}\right) \sigma^2 (X^TX)^{-1}
	$$
	as claimed.
\end{proof}

\section{Utility lemmas}

\begin{lemma}[Gaussian tail bound]\label{lem:gaussian-tail}
	Let $X\sim \cN(0,1)$. Then 
	$$
	\P(|X| >\epsilon) \leq \frac{2e^{-\epsilon^2/2}}{\epsilon}.
	$$
\end{lemma}
\begin{lemma}[$\chi^2$-distribution tail bound {\citep[Lemma 1]{laurent2000adaptive}}]
	Let $X$ follows a $\chi^2$ distribution with $k$ degree of freedom, then for all $t>0$, we have
	\begin{align*}
	&\P(X -k \geq 2\sqrt{k t} + 2t) \leq e^{-t},\\
	&\P(k-X > 2\sqrt{kt})  \leq e^{-t}.
	\end{align*}
\end{lemma}

\begin{lemma}[Tail bound to $(\epsilon,\delta)$-DP conversion]\label{lem:tailbound2DP}
	Let $\epsilon(\theta) = \log(\frac{p(\theta)}{p'(\theta)})$ where $p$ and $p'$ are densities of $\theta$. If 
	$$
	\P(|\epsilon(\theta)| > t) \leq \delta
	$$
	then for any measurable set $\cS$
	$$
	\P_p(\theta \in \cS)  \leq e^t \P_{p'}(\theta \in \cS) + \delta.
	$$
	and 
	$$\P_{p'}(\theta \in \cS)  \leq e^t \P_{p}(\theta \in \cS) + \delta$$
\end{lemma}
\begin{proof}
	Since $\log(\frac{p(\theta)}{p'(\theta)}) = - \log(\frac{p'(\theta)}{p(\theta)})$ and the tail bound is two-sided. It suffices for us to prove just one direction. Let $E$ be the event that $|\epsilon(\theta)| > t$.
	$$
	\P_p(\theta \in \cS)  =  \P_p(\theta \in \cS \cup E^c) + \P_p(\theta \in \cS \cup E) \leq  \P_{p'}(\theta \in \cS \cup E) e^t + \P_p(\theta\in E)\leq e^t\P_{p'}(\theta \in \cS) + \delta.
	$$
\end{proof}

\begin{lemma}[Weyl's eigenvalue bound {\citep[Theorem~1]{stewart1998perturbation}}]\label{lem:weyl}
	Let $X,Y,E\in \R^{m\times n}$, w.l.o.g., $m\geq n$. If $X-Y = E$, then $|\sigma_i(X)-\sigma_i(Y)|\leq \|E\| $ for all $i=1,...,n$.
\end{lemma}

\begin{lemma}[Stability of smooth learning problems, Lemma~14 of \citep{wang2017per}]\label{lem:smoothlearning}
	Assume $\ell$ and $r$ be differentiable and their gradients be absolute continuous.
	Let $\hat{\theta}$ be a stationary point of $\sum_i \ell(\theta,z_i)  + r(\theta)$, $\hat{\theta}'$ be a stationary point $\sum_i \ell(\theta,z_i) + \ell(\theta,z) + r(\theta)$ and in addition, let $\eta_t = t\hat{\theta} +(1-t)\hat{\theta}'$ denotes the interpolation of $\hat{\theta}$ and $\hat{\theta}'$.
	Then the following identity holds:
	\begin{align*}
	\hat{\theta}-\hat{\theta}' &= \left[\int_0^1\left(\sum_i\nabla^2\ell(\eta_t,z_i) +\nabla^2\ell(\eta_t,z) + \nabla^2r(\eta_t)\right) dt \right]^{-1} \nabla \ell(\hat{\theta},z)\\
	&= -\left[\int_0^1\left(\sum_i\nabla^2\ell(\eta_t,z_i) + \nabla^2r(\eta_t)\right) dt \right]^{-1} \nabla \ell(\hat{\theta}',z).
	\end{align*}
\end{lemma}

\section{$(\epsilon,\delta)$-DP calibration of \OPS{} for linear regression.}\label{app:epsdelta_OPS}
This appendix describes the details of how we implement the non-adaptive version of \OPS{} as a baseline. 

\OPS{} was proposed as a $\epsilon$-pure-DP mechanism via the use of the exponential mechanism. In this paper, we are working with $(\epsilon,\delta)$-DP  and it is only fair to compare to a version of \OPS{} with $(\epsilon,\delta)$-DP. Such guarantees are studied by \citet[Chapter 5]{mir13} and later by \citet{minami2016differential}, but neither can be straightforwardly and satisfactorily applied to the linear regression problem.

\citet{minami2016differential} requires that the loss function is Lipschitz. Linear regression is not Lipschitz unless we constraint $|\Theta|$ as in Assumption A2 just like for \OBP{} then it becomes Lipschitz. With appropriate choice of $\lambda$ and $\gamma$ and using ideas in Section~\ref{sec:adaptive_choice_of_lambda}. Unfortunately, unlike \OBP{}, \OPS{} is not an optimization based method. Sampling from the posterior distribution subject to the additional constraint requires techniques such as rejection sampling, which we find very costly, and prone to numerical issues.

\citet{mir13} does not require an explicit constraint on the parameter space. Instead, they use a large regularization parameter $\lambda$, so that  with probability $1-\delta$ over the distribution of the \OPS mechanism, the output is not too much larger than Ridge regression solution, which effectively produces a constraint on the domain. Then they apply an exponential mechanism-based argument after conditioning on this high-probability event. See Section~5.4.1 of \citep{mir13} for details. Unfortunately, this approach yields a suboptimal rate under $(\epsilon,\delta)$-DP, which depends linearly in $d$ rather than the optimal $\sqrt{d}$ dependence.

The pDP analysis of linear regression of \citet{wang2017per} suggests that we do not actually need global Lipschitz constant, instead the local Lipschitz constant at $\theta^*_\lambda$ is sufficient for us to obtain differential privacy. For any data set $(X,\mathbf{y})$, we can show that 
\begin{equation}\label{eq:effective_domainbound_from_regularization}
\|\theta^*_\lambda\|  \leq \|(X^TX+\lambda I)^{-1}X^T\|_2\sqrt{n}\| \cY\| \leq \min\left\{ \frac{\sqrt{n}\|\cY\| }{\sqrt{2\lambda}}, \frac{n\|\cX\|\|\cY\|}{ \lambda} \right\}.
\end{equation}
The local Lipschitz constant at $\theta^*_\lambda$ is therefore smaller than
$$
\frac{\sqrt{n}\|\cX\|^2\|\cY\| }{\sqrt{2\lambda}} + \|\cX\|\|\cY\| =  \|\cX\|\|\cY\| (  \frac{\sqrt{n}\|\cX\|}{\sqrt{2\lambda}}+1) =: L(\lambda).
$$

Apply Remark~\ref{rmk:pDP_eps} with the above Lipschitz constant upper bound and also take  $\lambda_{\min }=0$, we get a pDP guarantee for any pairs of adjacent data sets, which by definition, upgrades into a  DP guarantee.  In other word, we can achieve a prescribed $(\epsilon,\delta)$-DP by choosing choose any $(\lambda,\gamma)$ such that they obey 
$$
\epsilon \leq \sqrt{\frac{\gamma L(\lambda)^2\log(2/\delta)}{\lambda}} +  \frac{\gamma L(\lambda)^2}{ 2(\lambda + \|\cX\|^2)} + \frac{(1+ \log(2/\delta))\|\cX\|^2}{2\lambda}.
$$
There are many ways of doing it.  If we fix $\lambda >  \frac{(1+ \log(2/\delta))\|\cX\|^2}{2\epsilon}$, then we can calibrate $\gamma$ to achieve any $(\epsilon,\delta)$-DP guarantee for any $(\epsilon,\delta)$. If we instead fix $\gamma$ so that we have a comfortable level of variance, then similarly we can calibrate $\lambda$ to achieve any $(\epsilon,\delta)$-DP guarantee for any $(\epsilon,\delta)$.

Specifically, we will experiment with the following three approaches:
\begin{enumerate}
	\item \OPS{}-Diffuse:  Take $\lambda = \frac{(1+ \log(2/\delta))\|\cX\|^2}{\epsilon}$ and calibrate $\gamma$.
	\item \OPS{}-Concentrated:   Take $\gamma = 1$ and calibrate $\lambda$.
	\item \OPS{}-Balanced: Choose $\lambda$ to minimize the prediction accuracy upper bound that we have from Section~\ref{sec:adaptive_choice_of_lambda}
	\begin{equation}\label{eq:ops_suboptimality_bound}
	F(\hat(\theta)) - F(\theta^*) \leq \frac{C_1(\min(\epsilon,\sqrt{\epsilon}),\delta,\varrho,d) L(\lambda)^2 }{\lambda}   + \lambda B^2
	\end{equation}
	subject to $\lambda \geq \frac{(1+ \log(2/\delta))\|\cX\|^2}{\epsilon}$ \footnote{Note that this upper bound is obtained for  $\gamma = \frac{\lambda\min(\epsilon^2,\epsilon)}{4\log(2/\delta) L(\lambda)^2}$ and assuming $\|\theta^*\|\leq B$.}.  Once $\lambda$ is chosen, we then calculate $\gamma$ properly using the ``diffused'' approach given this $\lambda$.
	In \eqref{eq:ops_suboptimality_bound}, the function $C_1$ is defined in \eqref{eq:C1} and we choose $\varrho=0.05$.  $B$ is a more delicate hyperparameter since there isn't an upper bound of $\|theta^*\|$ that holds uniformly for all data sets. We will be using $B=1$ as we are being optimistic.
	\item \OPS{}-Conservative: An alternative approach that avoids choosing $B$ is to use $\|\theta^*_\lambda\|  \leq  \|\theta^*\| $ so that the minimizer of the upper bound $\lambda$ does not depend on $\|\theta^*\|$.
\end{enumerate}

\begin{figure}
	\centering
	\includegraphics[width=0.4\textwidth]{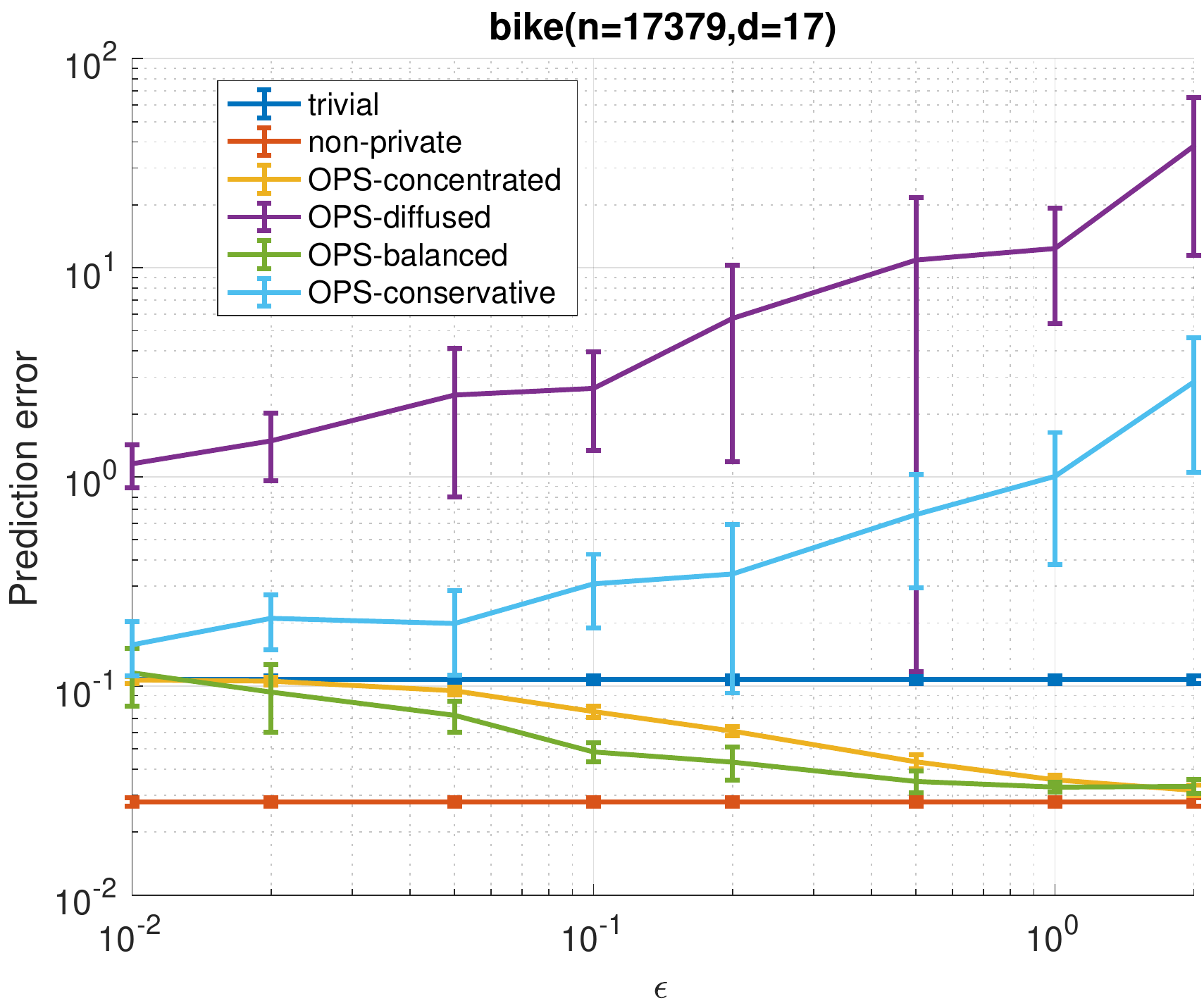}
	\includegraphics[width=0.4\textwidth]{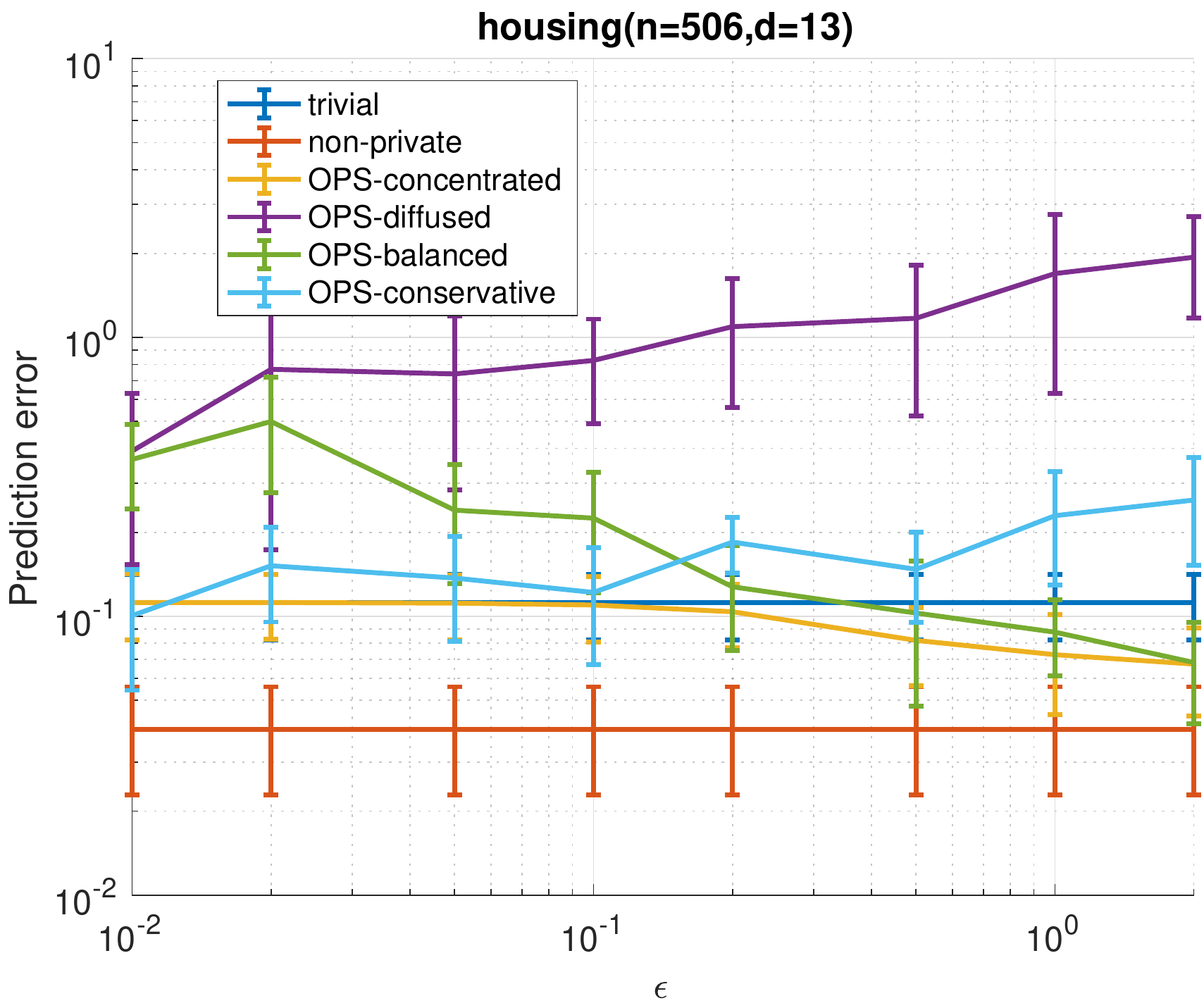}
	\caption{Comparison of the utility of $(\epsilon,\delta)$-DP \OPS{} using four different ways of calibrating noise to privacy. The results show that how we calibrate noise plays an important role in the utility of the algorithms. }\label{fig:compare_OPS}
\end{figure}

In our experiments, we find that no single approach dominates the others. In general, we find that the ``concentrated'' approach and the ``balanced'' approach with $B=1$ work significantly better than the ``diffused'' and the ``conservative'' approaches (see Figure~\ref{fig:compare_OPS} for details).
The experimental results with legend label ``\OPS{}'' in Figure~\ref{fig:UCI_dataset}, Table~\ref{tab:uci_eps=0.1} and Table~\ref{tab:uci_eps=1} are for the ``balanced' approach.

Below, we provide an error bound of the balanced approach. 
\begin{proposition}\label{prop:eps_delta_OPS_unbounded}
	Assume $\|\theta^*\| \asymp B$ on this specific data set. Then \OPS{} in unbounded domain with 
	$
	\gamma = \frac{\epsilon^2 \lambda}{ 4 \log(2/\delta) L(\lambda)^2 }
	$
	and $
	\lambda =\left(  \frac{C_1(\epsilon,\delta,\varrho,d)  \|\cX\|^4\|\cY\|^2 n}{B^2} \right)^{1/3}.
	$
	obeys $(\epsilon,\delta)$-DP, and also 
	$$
	F(\hat{\theta})  - F(\theta^*) \leq  O\left(  \frac{ d^{1/3}n^{1/3}\log(2/\delta)^{1/3} \|\cX\|^{4/3}\|\cY\|^{2/3} \|\theta^*\|^{4/3}}{\epsilon^{2/3}}  \right).  
	$$
\end{proposition}
\begin{proof}
	The result follows straightforwardly by substituting the $L(\lambda)$ and our choice of $\lambda,\gamma$ into \eqref{eq:ops_suboptimality_bound} checking that $\lambda$'s choice balances the two terms.
\end{proof}

\end{document}